\documentclass[12pt,twoside,a4paper]{article}
\usepackage{mathrsfs}
\usepackage[utf8]{inputenc}
\usepackage[T1]{fontenc}
\usepackage{amsmath,amssymb,amsthm,mathtools}
\usepackage{enumitem}
\usepackage{setspace}
\usepackage[margin=1.3in]{geometry}
\usepackage{xcolor}
\usepackage{hyperref}
\usepackage{cleveref}
\usepackage{lmodern} 
\usepackage{eucal} 

\usepackage{microtype}
\usepackage{booktabs}
\usepackage{thmtools}

\newcommand{\mc}{\mathcal}

\definecolor{spring}{RGB}{174, 83, 139}

\hypersetup{colorlinks, breaklinks,
linkcolor={[RGB]{174,83,139}},
citecolor={[RGB]{97,62,173}},
urlcolor={[RGB]{115,135,195}}}
\theoremstyle{definition}
\newtheorem{definition}{Definition}[section]
\newtheorem{example}[definition]{Example}

\theoremstyle{plain}
\newtheorem{theorem}[definition]{Theorem}
\newtheorem{lemma}[definition]{Lemma}
\newtheorem{proposition}[definition]{Proposition}
\newtheorem{corollary}[definition]{Corollary}

\newtheorem{assumption}[definition]{Assumption}

\theoremstyle{remark}
\newtheorem{remark}[definition]{Remark}

\newcommand{\E}{\mathbb{E}}
\newcommand{\EE}{\mathbb{E}}

\newcommand{\Fcal}{\mathcal{F}}

\newcommand{\Acal}{\mathcal{A}}
\newcommand{\Ecal}{\mathcal{E}}
\DeclareMathOperator{\cl}{cl}

\usepackage{pifont}
\usepackage{fancyhdr}
\definecolor{mblue}{RGB}{60, 88 148}
\definecolor{andreagreen}{RGB}{204,102, 0}

\newcommand{\mind}{\mathfrak{m}}

\usepackage{dsfont}
\usepackage[authoryear]{natbib}
\usepackage{etoc}

\usepackage{wrapfig}

\usepackage[most]{tcolorbox}

\newtcolorbox{examplebox}{
  colback=mblue!4,
  colframe=black,
  boxrule=0.6pt,
  arc=3pt,
  left=6pt,
  right=6pt,
  top=6pt,
  bottom=6pt
}
\usepackage{tikz}
\usetikzlibrary{arrows.meta}
\usetikzlibrary{shapes}
\usetikzlibrary{shapes.geometric} 
\usetikzlibrary{positioning}
\newcommand{\tgt}{\mathrm{tgt}}

\newcommand{\dperp}{\perp\!\!\!\perp}

\begin{document}
\addtocontents{toc}{\protect\setcounter{tocdepth}{-5}}
\title{\vspace{-1cm}
\Huge \bfseries
A Mathematical Theory of Understanding}
\author{{\large{Bahar Ta\c{s}kesen}}\\ \textit{\normalsize University of Chicago, Booth School of Business}}
\maketitle

\begin{abstract}
Generative AI has transformed the economics of information production, making explanations, proofs, examples, and analyses available at very low cost. Yet the value of information still depends on whether downstream users can absorb and act on it. A signal conveys meaning only to a learner with the structural capacity to decode it: an explanation that clarifies a concept for one user may be indistinguishable from noise to another who lacks the relevant prerequisites. This paper develops a mathematical model of that learner-side bottleneck.
We model the learner as a \textit{mind}, an abstract learning system characterized by a prerequisite structure over concepts. A mind may represent a human learner, an artificial learner such as a neural network, or any agent whose ability to interpret signals depends on previously acquired concepts. Teaching is modeled as sequential communication with a latent target. Because instructional signals are usable only when the learner has acquired the prerequisites needed to parse them, the effective communication channel depends on the learner's current state of knowledge and becomes more informative as learning progresses. The model yields two limits on the speed of learning and adoption: a structural limit determined by prerequisite reachability and an epistemic limit determined by uncertainty about the target.
The framework implies threshold effects in training and capability acquisition. When the teaching horizon lies below the prerequisite depth of the target, additional instruction cannot produce successful completion of teaching; once that depth is reached, completion becomes feasible. This generates non-concave returns to training effort and implies that spreading scarce instructional resources evenly can yield lower output than concentrating them on fewer workers or users. Across heterogeneous learners, a common broadcast curriculum can be slower than personalized instruction by a factor linear in the number of learner types.
\end{abstract}

\section{Introduction}\label{sec:intro}
The value of a piece of information depends on the existence of a mind that can decode it. This is evident in ordinary learning: an explanation means nothing to a listener who lacks the background to follow it, and a lecture conveys nothing to a student who cannot parse its content. Information that cannot be absorbed by the intended learner is, in a precise sense, noise. Understanding therefore cannot be reduced to the accumulation of information alone. Whether a signal carries usable information is not an intrinsic property of the signal itself, but a relation between the signal and the conceptual structure of the mind that receives it.

Over the past century, the cost of producing and distributing information has fallen by orders of magnitude, from printed encyclopedias to digital repositories to, most recently, generative AI systems capable of producing explanations, proofs, and worked examples on demand. As the supply of machine-generated information expands, the bottleneck shifts from production to absorption: the ability of downstream users to parse, interpret, and act on what is produced. Whether a signal carries usable information depends on the learner's ability to decode it. An explanation that conveys meaning to one user may be indistinguishable from noise to another who lacks the relevant prerequisites.

This paper develops a mathematical model of that learner-side bottleneck through a formal theory of understanding. We do not attempt to model every feature of cognition, such as analogy, abstraction, forgetting, or semantic interpretation. Instead, we ask a narrower structural question: given a learner with a fixed prerequisite architecture, which concepts are understandable in principle, through which intermediate states can the learner move, and what limits the speed at which instruction can bring the learner to a target?

Our starting point is a formal model of a \emph{mind}. We use the term {mind} for a learning system whose ability to interpret new signals depends on what has already been acquired. The same formal object can represent a human learner, an artificial learner such as a neural network, or any agent whose decoding power is shaped by prerequisite structure. Formally, a mind consists of a concept space together with an axiom set and a family of finitary expansion rules specifying which concepts become accessible once their prerequisites have been acquired. These rules induce an understanding horizon, describing what is reachable in principle from the axioms, and a family of reachable acquired concept sets, describing the intermediate states through which a learner can progress by successive prerequisite-respecting steps. Under a finite-horizon assumption, we show that this reachable family forms a learning space above the axiom set, equivalently an antimatroid, and conversely that every such structure admits a representation by an appropriate mind.

To study the operational consequences of this structure, we model teaching as sequential communication with a latent target concept. The teacher knows the realized target but the learner does not. Instructional signals are filtered through a prerequisite-gated parser induced by the mind: a signal is usable only when its target concept is currently ordered for the learner, and otherwise collapses to a common null observation. The effective learner-side channel is therefore not fixed in advance. It depends on the learner's current knowledge state and changes as instruction proceeds. The same raw broadcast may convey usable information to one learner while collapsing to noise for another. We call this phenomenon the \emph{relativity of randomness}.

This state dependence creates two distinct obstacles to fast teaching. The first is \emph{structural}: before a target can be acquired, the learner must move through prerequisite-respecting states until the target becomes currently parseable. The second is \emph{epistemic}: the learner must infer which target the teacher intends. Our main quantitative result combines these two bottlenecks into a general lower bound on teaching time. Expected completion time must clear both a structural barrier, determined by the shortest valid route to the target, and an epistemic barrier, determined by the cumulative usable information that can pass through the learner-side channel. The structural barrier can be dominant, but the information-theoretic layer remains essential for characterizing when and how instruction becomes usable. In our model, once the prerequisite structure makes the target reachable, one additional signal is enough for identification.

This framework leads to several consequences. Acquiring prerequisites does more than add concepts: in the sense of Blackwell, it refines the statistical experiment through which later instruction about the target is interpreted. This structural change has operational implications for teaching. For deterministic targets, fixed-horizon teaching problems exhibit discontinuous structural thresholds: completion probability jumps from zero to one when the teaching horizon reaches the structural distance to the target concept, implying non-concave returns to instructional time and simple failures of uniform resource allocation. The same structural logic also shapes multi-learner settings. Across heterogeneous learners, teaching with a common broadcast curriculum can be strictly slower than personalized instruction by a factor linear in the number of learner types.

\paragraph{Related literature.}
The paper draws on several literatures but differs from each in a specific
way. At a broad level, our question is how the structure of a learner limits
the usable flow of information. This connects combinatorial models of learning,
information theory, teaching and machine learning, and models of skill
formation, but the present framework combines these ingredients in a way that
is specific to prerequisite-gated understanding.

The combinatorial study of feasible learning states originates with knowledge
space theory
\citep{ref:doignon1999knowledge,ref:doignon2015knowledge},
where the family of feasible states is taken as a primitive. Independently,
\citet{korte1991greedoids} arrived at the same mathematical
structure, antimatroids, from the perspective of combinatorial optimization.
We recover this structure from a different starting point: a generative model
of a mind specified by axioms and finitary expansion rules. The equivalence
(\Cref{thm:representation-learning-space}) shows that the two viewpoints are
formally interchangeable, but the generative formulation connects the
combinatorial structure directly to closure, derivability, and the teaching
bounds developed in the paper.

Shannon's information theory \citep{shannon1948} studies channels whose
input-output relationship is fixed. In our framework, by contrast, the
learner's parsing map induces an effective channel whose output alphabet
depends on the learner's current acquired state. Blackwell's comparison of
experiments \citep{ref:blackwell1951comparison,ref:blackwell1} provides
the natural language for this dependence: we show that the parsed experiment
induced by a larger acquired state Blackwell-dominates the one induced by a
smaller state.

In computational learning theory, \citet{goldman1995complexity} introduce
teaching dimension as a combinatorial measure of how many labeled examples
suffice to identify a target concept within a learner class. Our setting is
different: the main constraint is not only identification, but whether the
learner can parse target-relevant signals at all given its current
prerequisites. More recent work in machine teaching studies settings in which a
single teacher must instruct multiple
heterogeneous learners with a common teaching sequence; for example,
\citet{zhu2017nolearnerleftbehind} show that common teaching can be strictly
harder than individualized teaching, and \citet{zhu2018overview} survey the
broader landscape. Our broadcast impossibility result
(\Cref{thm:impossibility}) differs in the source of the penalty: it is driven
by prerequisite-gated decodability and the geometry of reachable acquired
states, rather than by differences in algorithmic update rules across learners.
The term \emph{curriculum} also appears in machine learning, where it typically
refers to the ordering of training examples from easy to hard
\citep{bengio2009curriculum}. There the object being shaped is the optimization
trajectory of a parametric model; here it is the sequence of
prerequisite-respecting states through which a structured learner can move.

The threshold and allocation results in \Cref{sec:structural-limits} are also related in spirit to models of human-capital accumulation \citep{ref:becker1964,ref:benporath1967,cunha2007technology}. Those models study how current investment affects future skill formation, often through complementarity across stages. Our mechanism is different. In our framework, missing prerequisites create structural thresholds: below the relevant structural depth, completion is impossible regardless of strategy, whereas beyond that threshold completion becomes feasible. The resulting non-smoothness comes from prerequisite-gated decodability rather than from an exogenous production technology.
The state-dependent information constraint
also connects to rational inattention \citep{sims2003}: both frameworks study
limits on usable information, but in rational inattention the bottleneck is
imposed through an explicit information cost, whereas here it arises
endogenously from the prerequisite structure of the mind.

Finally, the observation that absorptive capacity limits the value of information connects naturally to emerging work on the economic implications of AI-generated content. As generative models reduce the cost of producing explanations, examples, and analyses, the central question becomes who can make use of the resulting output. In our framework, this bottleneck arises from the prerequisite structure of the learner, which determines which generated signals carry usable information and which collapse to noise.

\medskip
\paragraph{Notation and conventions.}
We write $\Delta(\Omega)$ for the set of probability distributions on a finite or countable set $\Omega$. Unless stated otherwise, all logarithms are taken to base~$2$; accordingly, entropy and mutual information are measured in bits. For a set $\mathcal S$, we denote its cardinality by $|\mathcal S|$ and its power set by $2^{\mathcal S}$. Finally, $\delta_x$ denotes the point mass at $x \in \mathcal S$.

\section{Understanding as a Closure System}\label{sec:axioms}

What does it mean for a learner to ``understand'' something?
A child who knows addition can follow a multiplication lesson built on repeated addition; one who lacks addition cannot follow that explanation. The same words carry information for one mind and are noise for another.
Understanding, in this sense, is not an isolated state but a structured dependency: each concept requires certain prerequisites, and those prerequisites may themselves depend on prior knowledge.

To formalize this idea, we introduce a primitive notion of concept and a nonempty \emph{concept space}, whose elements represent the conceptual units under consideration. 
A mind is then specified by two objects: a set of axioms and a set of expansion rules. Axioms are concepts taken as given, requiring no further justification. Each expansion rule states that mastery of a specific finite set of concepts, referred to as its prerequisites, unlocks a new concept.
Different minds may share the same concept space yet differ in their axioms or expansion rules. In that case, the order in which concepts become learnable differs, capturing the familiar observation that individuals with different backgrounds require different learning paths.

Given a mind, the expansion rules induce a closure operator: starting from any set of known concepts, iteratively apply every expansion rule whose prerequisites are satisfied until no new concepts are added. The resulting closure operator satisfies extension, monotonicity, and idempotence, the standard closure axioms. These are not merely formal conveniences. Extension encodes that knowledge is never lost by derivation. Monotonicity encodes that knowing more can only enlarge what is derivable. Idempotence encodes that once all consequences have been drawn, further application changes nothing. Any reasonable notion of logical or conceptual consequence must satisfy these properties.
The closure framework provides the basic structural language in which the
notion of understanding will be formalized in the sections that follow.

We now formalize these ideas using closure operators from order theory.

\begin{definition}[Concept space]\label{def:concept-space}
A \emph{concept space} is a nonempty set $\mc C$ whose elements are \emph{concepts}.
\end{definition}

The concept space $\mc C$ is a modeling primitive: its elements may represent facts (``zebras are animals''), skills (``long division''), propositions (``the fundamental theorem of calculus''), or procedures (``how the simplex method works'') at any level of granularity. The framework is invariant to this choice. The modeler selects $\mc C$ in the same way an economist selects the state space in a decision problem or the type space in a mechanism design model: the choice determines which phenomena the model can express, but the theorems themselves do not depend on the particular interpretation.
The concept space $\mathcal C$ may be finite or countably infinite. When concepts admit
finite descriptions, they can be represented as finite strings over a finite
alphabet, and $\mc C$ can therefore be identified with a subset of that set.

\begin{definition}[Mind]\label{def:mind}
A \emph{mind} over a concept space $\mc C$ is a triple 
$\mind = (\mc C, \Acal_\mind, \mathcal{E}_\mind)$ where:
\begin{enumerate}[label=(\roman*)]
\item $\Acal_\mind \subseteq \mc C$ is a set of \emph{axioms},
\item $\mathcal{E}_\mind \subseteq 2^{\mc C}_{\mathrm{fin}} \times \mc C$ 
is a set of \emph{expansion rules}, where $2^{\mc C}_{\mathrm{fin}}$ denotes 
the collection of finite subsets of $\mc C$.
\end{enumerate}
\end{definition}
The axioms $\Acal_\mind$ are the concepts that the mind $\mind$ understands
\emph{a priori}: they require no prerequisites. Each expansion rule
$(\mc S, c) \in \Ecal_\mind$ states that if all concepts in the finite set
$\mc S$ are currently understood, then the concept $c$ becomes accessible.
The set $\mc S$ is referred to as the \emph{prerequisites} of $c$ under that rule.

{The expansion rules $\mathcal{E}_\mind$ describe the cognitive architecture of
the mind, that is, the wiring that determines what can be derived from what,
rather than propositions explicitly known by the learner. A rule in
$\mathcal E_\mind$ is not assumed to be something the learner can articulate;
instead, it specifies which concepts become accessible once the learner has
mastered the prerequisites. The \emph{teacher}, by contrast, may or may not
know $\mathcal{E}_\mind$. A teacher with full knowledge of the learner's rules
can tailor instruction to the learner's prerequisite structure, whereas a
teacher who is ignorant of the learner's type may have to resort to a common
broadcast and can then pay the price of universality
(\Cref{thm:impossibility}).

We impose one structural restriction on $\mathcal{E}_\mind$: each prerequisite
set is finite. Accordingly, the granularity of $\mc C$, \textit{i.e.}, what
counts as a single concept, should be chosen so that realistic explanations can
be modeled using finitely many prerequisites. Beyond this finitarity
requirement, the level of granularity is a modeling choice.}

\begin{remark}
\label{rmk:scope}
We do not model logical inconsistency or belief revision. Concepts are treated as abstract units, and understanding refers to accessibility under a prerequisite structure rather than to semantic truth. This is a deliberate modeling choice, analogous to Shannon's separation of the engineering problem of communication from the semantic content of messages.
Accordingly, a concept in our framework may represent a true theorem, a useful heuristic, or even a widespread misconception. The theory is invariant to this distinction: the teaching bounds depend only on the dependency structure induced by the prerequisite rules and on the information geometry of the teaching interaction, not on the truth value of the concepts themselves.
\end{remark}

\begin{example}[Two minds learning arithmetic]\label{ex:arithmetic-minds}
Let $\mc C = \{a,b,c,d\}$ with the informal readings 
$a$ = counting, $b$ = addition, $c$ = spatial arrays, $d$ = multiplication.

\emph{Mind~1 (algorithmic learner).}
Axioms $\Acal_1=\{a\}$. The expansion rule set $\mathcal E_1$ consists of 
$\{a\}\Rightarrow b$,
$\{b\}\Rightarrow c$,
$\{b,c\}\Rightarrow d$.
This mind first understands addition from counting, then understands spatial arrays through repeated addition, and finally grasps multiplication once it combines repeated addition with the array representation.

\emph{Mind~2 (visual learner).}
Axioms $\Acal_2=\{a\}$. The expansion rule set $\mathcal E_2$ consists of $
\{a\}\Rightarrow c$, $
\{c\}\Rightarrow b$, 
$\{b,c\}\Rightarrow d$.
This mind first understands spatial arrays from counting objects arranged in space, then understands addition by combining arrays, and finally reaches multiplication through the same rule $\{b,c\}\Rightarrow d$.
\end{example}

Both minds in Example~\ref{ex:arithmetic-minds} share the same concept space and the same axioms, and both can eventually derive all four concepts, but the order in which concepts become available differs. A concept that one mind derives early may come late for the other. This is the formal expression of \emph{relativity of understanding}: individuals with different cognitive architectures can arrive at the same body of knowledge through fundamentally different paths. We will revisit this example throughout the paper.

Example~\ref{ex:arithmetic-minds} is about learning mathematics, but the framework applies to any domain in which understanding has prerequisite structure. The next example illustrates this.

\begin{example}[Two minds learning text editing on a computer]\label{ex:editor}
Let $\mc C=\{t,s,k,e\}$ with the informal readings
$t$ = typing text, 
$s$ = selecting (highlighting) text,
$k$ = keyboard shortcuts,
$e$ = efficient editing.

\emph{Mind~3 (mouse-first).}
The axiom set is $\Acal_3=\{t\}$.
The expansion rule set $\mathcal E_3$ consists of $
\{t\}\Rightarrow s$,
$\{s\}\Rightarrow k$,
$\{s,k\}\Rightarrow e$.
This learner first acquires text selection from typing, then acquires keyboard shortcuts once selection is understood, and finally reaches efficient editing once both selection and shortcuts are available.

\emph{Mind~4 (shortcut-first).}
The axiom set is $\Acal_4=\{t\}$.
The expansion rule set $\mathcal E_4$ consists of $
\{t\}\Rightarrow k$,
$\{k\}\Rightarrow s$,
$\{s,k\}\Rightarrow e$.
This learner first acquires keyboard shortcuts from typing, then acquires selection through shortcut-based interaction, and finally reaches efficient editing once both selection and shortcuts are available.

Both minds share the same concept space and the same axiom set, and both can ultimately reach $e$. However, their prerequisite structures differ: in Mind~3, selection unlocks shortcuts, whereas in Mind~4, shortcuts unlock selection. The final rule $\{s,k\}\Rightarrow e$ is shared, but the paths by which its prerequisites are acquired are different.
\end{example}

The expansion rules admit a combinatorial interpretation.
They form a directed hypergraph \citep{ref:berge1984hypergraphs} in which each rule $(\mc S,c)$ is a
hyperedge from the prerequisite set $\mc S$ to the concept $c$. 

\begin{definition}[One-step expansion]\label{def:one-step}
For a mind $\mind$ and a set $\mc K \subseteq \mc C$ of currently known
concepts, define the \emph{one-step expansion} by $
\Phi_\mind(\mc K)
=
\mc K
\cup
\{c\in\mc C : \exists \mc S\subseteq \mc K \text{ such that } (\mc S,c)\in\mathcal E_\mind\}$.
\end{definition}
For Mind~1 in \Cref{ex:arithmetic-minds}, start from $\mc K=\{a\}$.
The rule $\{a\}\Rightarrow b$ fires, since $\{a\}\subseteq \{a\}$, and therefore $
\Phi_1(\{a\})=\{a,b\}$.
Applying the operator again, the rule $\{b\}\Rightarrow c$ fires, whereas
$\{b,c\}\Rightarrow d$ does not, since $c\notin\{a,b\}$. Thus $
\Phi_1(\{a,b\})=\{a,b,c\}$.
Applying the operator once more, the rule $\{b,c\}\Rightarrow d$ now fires, so $
\Phi_1(\{a,b,c\})=\{a,b,c,d\}$.
A further application produces no new concepts, so $\{a,b,c,d\}$ is a fixed point of $\Phi_1$.

{Note that $\Phi_\mind$ is \emph{extensive}: by definition, the union in \Cref{def:one-step} includes $\mc K$ itself, so $\mc K \subseteq \Phi_\mind(\mc K)$ for every $\mc K \subseteq \mc C $. We use this property freely throughout.}
{Two further properties of $\Phi_\mind$ ensure that repeated application yields a well-defined closure. Monotonicity guarantees that knowledge never shrinks, while preservation of directed unions ensures that no concept appears only “at the limit”: whenever a concept is derivable from the union of an increasing family of knowledge states, it is already derivable at some stage in that family.}

\begin{lemma}[Monotonicity]\label{lem:mono}
If $\mc K \subseteq \mc K'$, then $\Phi_\mind(\mc K) \subseteq \Phi_\mind(\mc K')$.
\end{lemma}

\begin{definition}\label{def:closure}
For a mind $\mind$ and a set $\mc K \subseteq \mc C $, the \emph{understanding closure} of $\mc K$ is the smallest fixed point of $\Phi_\mind$ containing $\mc K$: $
\cl_\mind(\mc K) = \bigcap \{\mc F \subseteq \mc C  :\mc  K \subseteq \mc F \text{ and } \Phi_\mind(\mc F) = \mc F\}$.
The \emph{understanding {horizon}} of mind $\mind$ is $\mc U_\mind = \cl_\mind(\Acal_\mind)$.
\end{definition}
For a set map $\Phi_\mind:2^{\mc C }\to 2^{\mc C }$, a subset $\mc F\subseteq \mc C $ is a \emph{fixed point} if $\Phi_\mind(\mc F)=\mc F$. Fixed points are partially ordered by set inclusion. Given $\mc K\subseteq \mc C $, a fixed point $\mc F^\star$ is the \emph{least fixed point containing $\mc K$} if $\mc K\subseteq \mc F^\star$ and $\mc F^\star\subseteq \mc F$ for every fixed point $\mc F$ with $\mc K\subseteq \mc F$. In our setting, $\cl_\mind(\mc K)$ is defined as the intersection of all fixed points containing $\mc K$, hence it is precisely the least knowledge state that contains $\mc K$ and is closed under all expansion rules of mind $\mind$. In particular, the understanding horizon $\mc U_\mind=\cl_\mind(\Acal_\mind)$ is the least fixed point containing the axiom set $\Acal_\mind$, and therefore represents the set of all concepts \emph{potentially accessible} to $\mind$, that is, the theoretical horizon reachable from its axioms under its expansion rules.

{\begin{remark}\label{rmk:fixed-architecture}
It is natural to ask whether a teacher can impart new expansion rules, or how a concept unteachable to a toddler (\textit{e.g.}, abstract algebra) eventually becomes learnable years later. In our framework, methods and techniques that informally feel like ``rules'', such as the chain rule in calculus or \emph{modus ponens} in logic, are modeled as \emph{concepts} $c \in \mc C $. They are things a learner can be taught. Once mastered, they serve as prerequisites that unlock downstream concepts via the mind's existing expansion rules.
The expansion rules $\mathcal{E}_\mind$ and axioms $\Acal_\mind$ themselves are not teachable: they represent the learner's fixed cognitive architecture, sensory baseline, or developmental stage over the timescale of a teaching interaction. A concept is strictly unteachable ($c \notin \mc U_\mind$) when this architecture cannot bridge the gap from the axioms. If a learner eventually grasps a concept that was structurally inaccessible to their earlier self, we model this not as a teaching event, but as \emph{cognitive development}: a transition into a new mind $\mind'$ with richer axioms $\Acal_{\mind'}$, richer expansion rules $\mathcal{E}_{\mind'}$, or both. Our theory bounds the fundamental limits of \emph{teaching} a fixed architecture; the long-term \emph{development} of the architecture itself is a separate process.
\end{remark}}

\begin{proposition}[Existence and characterization]\label{prop:closure-char}
For any mind $\mind$ and any $\mc K \subseteq \mc C $:
\begin{enumerate}[nosep, label=\textnormal{(\roman*)}]
\item $\cl_\mind(\mc K)$ exists and is a fixed point of $\Phi_\mind$.
\item $\cl_\mind(\mc K) = \bigcup_{n=0}^{\infty} \Phi_\mind^n(\mc K)$, where $\Phi_\mind^0(\mc K) = \mc K$ and $\Phi_\mind^{n+1}(\mc K) = \Phi_\mind(\Phi_\mind^n(\mc K))$.
\item If $\mc C $ is finite, then $\cl_\mind(\mc K) = \Phi_\mind^N(\mc K)$ for some $N \leq |\mc C |$.
\end{enumerate}
\end{proposition}

The existence of a least fixed point in \Cref{prop:closure-char} follows from the Knaster-Tarski fixed point theorem \citep{tarski}; see also \citep{aliprantis} for a textbook treatment. We give a self-contained proof for completeness in \Cref{app:proofs}. 
\begin{proposition}[Axiomatic characterization of understanding]\label{prop:U-unique}
For a given mind $\mind$, the set $\mc U_\mind=\cl_\mind(\Acal_\mind)$ is the unique set $\mc U \subseteq \mc C$ satisfying:
\begin{enumerate}[nosep, label=\textnormal{(\roman*)}]
\item \emph{Axioms are understood:} $\Acal_\mind \subseteq \mc U$.
\item \emph{Closure under expansion:} if $(\mc S,c)\in\mathcal E_\mind$ and $\mc S\subseteq \mc U$, then $c\in \mc U$.
\item \emph{Minimality:} $\mc U$ is the smallest set satisfying \emph{(i)} and \emph{(ii)}.
\end{enumerate}
\end{proposition}
Property {(i)} of \Cref{prop:U-unique} ensures that the axioms belong to $\mc U_\mind$. 
Property {(ii)} enforces closure under expansion: whenever all prerequisites of a concept are already in the set, the concept itself must also belong to the set. 
Many subsets of $\mc C$ satisfy {(i)} and {(ii)}; the entire concept space $\mc C$ is a trivial example. Property {(iii)} removes this ambiguity by imposing minimality: $\mc U_\mind$ admits no proper subset that both contains the axioms and is closed under the expansion rules. 
Together, the three properties determine $\mc U_\mind$ uniquely. In this sense, understanding is completely determined by the axioms $\Acal_\mind$ and the expansion rules $\mathcal{E}_\mind$, with no additional degrees of freedom.

\subsection{Derivations and Equivalence}\label{sec:derivations}
The closure $\cl_\mind(\mc K)$ tells us \emph{which} concepts are reachable from $\mc K$, but not \emph{how} they are reached. 
A derivation makes the ``how'' explicit: it is a rooted tree whose nodes represent rule applications and base concepts, showing step by step why a concept lies in $\cl_\mind(\mc K)$. 
By \Cref{lem:deriv-finite}, every such tree is finite.

\begin{definition}[Derivation]\label{def:derivation}
A \emph{derivation of concept $c$ from $\mc K\subseteq \mc C$ in mind $\mind$} is a well-founded rooted tree whose nodes are labeled by concepts, satisfying:
\begin{enumerate}[nosep, label=(\roman*)]
\item Every node is either a \emph{base node} or a \emph{rule node}:
\begin{itemize}[nosep]
\item A \emph{base node} is a leaf (no children) labeled by a concept in $\mc K$.
\item A \emph{rule node} is labeled by a concept $c'$ and has children in bijection with a set $\mc S$ such that $(\mc S, c') \in \mathcal{E}_\mind$, with each child labeled by the corresponding element of $\mc S$.
\end{itemize}
\item The root is labeled by $c$. 
\end{enumerate}
We write $\mc K \vdash_\mind c$ if such a derivation exists.
\end{definition}

\begin{example}[Derivation trees for the two minds]\label{ex:arithmetic-derivations}
Continuing \Cref{ex:arithmetic-minds}, consider the derivation of $d$ (multiplication) from the axiom set $\mc K = \{a\}$. \Cref{fig:derivation-trees} shows the derivation trees for both minds. In each tree, the \emph{leaves} (bottom nodes, drawn as squares) are concepts from $\mc K$, which represents the starting knowledge. Each \emph{internal node} (drawn as a circle) is a concept derived by applying one expansion rule to its children (the nodes directly below it). The \emph{root} (top node) is the concept being derived.

Reading each tree bottom-up: Mind~1 derives $a \to b \to c \to d$ (addition before arrays); Mind~2 derives $a \to c \to b \to d$ (arrays before addition). The two derivation trees witness the same conclusion, namely that $
d \in \cl_1(\{a\}) \cap \cl_2(\{a\})$,
but through different intermediate paths. This provides a concrete instance of mind-relativity.
\end{example}

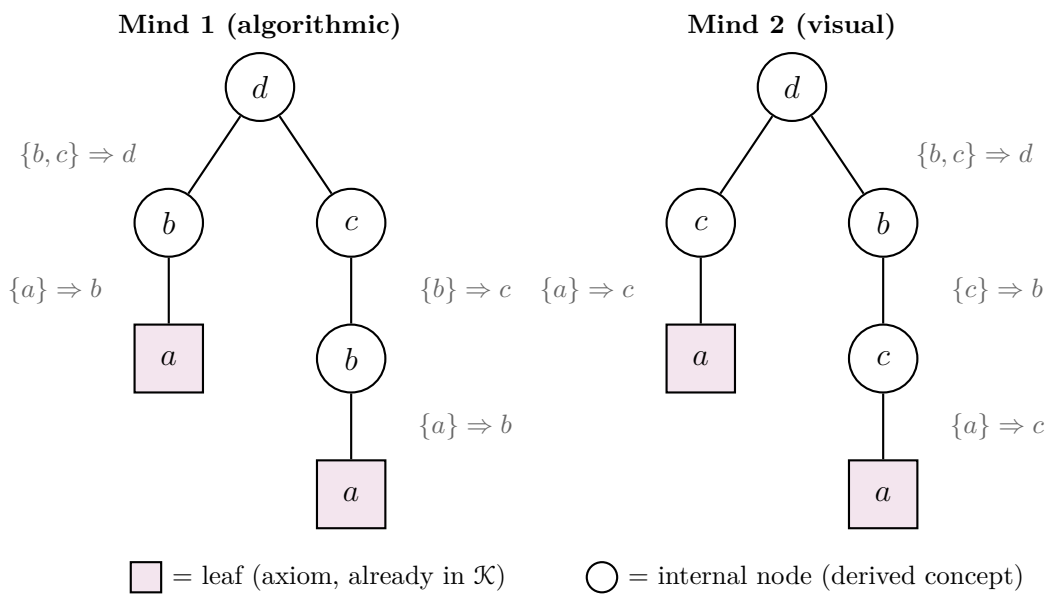
\begin{figure}[ht]
\centering
\begin{tikzpicture}[
    concept/.style={circle, draw, thick, minimum size=9mm, font=\normalsize},
    axiom/.style={rectangle, draw, thick, fill=spring!15, minimum size=9mm, font=\normalsize},
    rulebox/.style={font=\footnotesize, text=black!60},
    every edge/.style={draw, thick},
]
\node[font=\footnotesize \bfseries] at (-4.5, 0.8) {Mind~1 (algorithmic)};
\node[concept]  (d1) at (-4.5, 0)    {$d$};
\node[concept]  (b1) at (-5.7, -1.8) {$b$};
\node[concept]  (c1) at (-3.3, -1.8) {$c$};
\node[axiom]    (a1) at (-5.7, -3.6) {$a$};
\node[concept]  (b1r) at (-3.3, -3.6) {$b$};
\node[axiom]    (a1r) at (-3.3, -5.4) {$a$};
\draw[thick] (b1) -- (d1);
\draw[thick] (c1) -- (d1);
\draw[thick] (a1) -- (b1);
\draw[thick] (b1r) -- (c1);
\draw[thick] (a1r) -- (b1r);
\node[rulebox] at (-6.9, -0.9) {$\{b,c\} \Rightarrow d$};
\node[rulebox] at (-7.2, -2.7) {$\{a\} \Rightarrow b$};
\node[rulebox] at (-1.8, -2.7) {$\{b\} \Rightarrow c$};
\node[rulebox] at (-1.8, -4.5) {$\{a\} \Rightarrow b$};
\begin{scope}[xshift=-2cm]
\node[font=\footnotesize\bfseries] at (4.5, 0.8) {Mind~2 (visual)};
\node[concept]  (d2) at (4.5, 0)    {$d$};
\node[concept]  (c2) at (3.3, -1.8) {$c$};
\node[concept]  (b2) at (5.7, -1.8) {$b$};
\node[axiom]    (a2) at (3.3, -3.6) {$a$};
\node[concept]  (c2r) at (5.7, -3.6) {$c$};
\node[axiom]    (a2r) at (5.7, -5.4) {$a$};
\draw[thick] (c2) -- (d2);
\draw[thick] (b2) -- (d2);
\draw[thick] (a2) -- (c2);
\draw[thick] (c2r) -- (b2);
\draw[thick] (a2r) -- (c2r);
\node[rulebox] at (6.9, -0.9) {$\{b,c\} \Rightarrow d$};
\node[rulebox] at (1.8, -2.7) {$\{a\} \Rightarrow c$};
\node[rulebox] at (7.2, -2.7) {$\{c\} \Rightarrow b$};
\node[rulebox] at (7.2, -4.5) {$\{a\} \Rightarrow c$};
\end{scope}
\node[axiom, minimum size=4mm] at (-6, -6.5) {};
\node[font=\footnotesize, anchor=west] at (-5.8, -6.5) {= leaf (axiom, already in $\mc K$)};
\node[concept, minimum size=4mm] at (0, -6.5) {};
\node[font=\footnotesize, anchor=west] at (0.2, -6.5) {= internal node (derived concept)};
\end{tikzpicture}
\caption{Derivation trees for $d$ (multiplication) from $\mc K = \{a\}$ (counting) in the two minds of \Cref{ex:arithmetic-minds}. Each tree is read bottom-up: leaves are concepts already known; each internal node is derived from its children by the expansion rule shown alongside. The root~$d$ is the concept being derived. {Both trees witness that $d$ belongs to the corresponding understanding closure of $\{a\}$, but through different intermediate paths.}}
\label{fig:derivation-trees}
\end{figure}

{Derivations provide a constructive counterpart to the closure: if a concept belongs to $\cl_\mind(\mc K)$, there must exist a finite chain of rule applications that produces it. The following theorem confirms that the two characterizations are equivalent, that is, nothing belongs to the closure without a derivation, and every derivation stays within the closure.}

\begin{theorem}[Closure-derivability equivalence]\label{thm:closure-deriv}
For any mind $\mind$, any set $\mc K\subseteq\mc C$, and any concept $c\in\mc C$, $
c \in \cl_\mind(\mc K) \Longleftrightarrow \mc K \vdash_\mind c$.
\end{theorem}

The closure operator $\cl_\mind$ induced by a mind satisfies the usual closure axioms (extension, monotonicity, and idempotence). A closure operator that additionally satisfies a finitary property, namely that membership in the closure depends only on finitely many elements, is called algebraic~(see \Cref{def:alg-closure}).
\begin{theorem}[Algebraic closure equivalence]\label{thm:alg-equiv}
\leavevmode
\begin{enumerate}[nosep, label=\textnormal{(\roman*)}]
\item For any mind $\mind=(\mc C,\Acal_\mind,\mathcal E_\mind)$, the closure operator $
\cl_\mind:2^{\mc C}\to 2^{\mc C}$
is an algebraic closure operator on $\mc C$.

\item Conversely, for any set $\mc X$ and any algebraic closure operator $
f:2^{\mc X}\to 2^{\mc X}$,
there exists a rule set $\mathcal E\subseteq 2^{\mc X}_{\mathrm{fin}}\times \mc X$
such that, writing
\[
\Psi_{\mathcal E}(\mc K)
=
\mc K\cup
\{c\in\mc X:\exists \mc S\subseteq \mc K \text{ such that }(\mc S,c)\in\mathcal E\},
\]
one has, for every $\mc K\subseteq \mc X$,
\[
f(\mc K)
=
\bigcap\bigl\{
\mc F\subseteq \mc X :
\mc K\subseteq \mc F
\text{ and }
\Psi_{\mathcal E}(\mc F)=\mc F
\bigr\}.
\]
\end{enumerate}
\end{theorem}
\Cref{thm:alg-equiv} shows that finitary expansion-rule systems and algebraic
closure operators are equivalent ways of describing the same finitary
consequence relation. In particular, every finitary expansion-rule system
induces an algebraic closure operator, and conversely every algebraic closure
operator on a set $\mc X$ admits at least one, generally non-unique,
presentation by finitary expansion rules. Thus the rule-based component of a
mind should be understood not as additional structure beyond closure, but as a
presentation of an algebraic closure operator.

Conceptually, this separates structure from presentation. The expansion rules
describe one particular finite-premise decomposition of the underlying
consequence relation, while the intrinsic object is the algebraic closure
operator itself.

Concretely, let $\mc X$ be a nonempty set, let $f:2^{\mc X}\to 2^{\mc X}$ be an
algebraic closure operator, and let $\Acal\subseteq \mc X$ be a chosen set of
axioms. Choose any rule set $\mathcal E\subseteq 2^{\mc X}_{\mathrm{fin}}\times
\mc X$ whose induced closure operator is $f$, as guaranteed by
\Cref{thm:alg-equiv}(ii). Then $
\mind=(\mc X,\Acal,\mathcal E)$
is a mind whose induced closure operator is $f$, and whose
understanding is $
\mc U_\mind = f(\Acal)$.

Thus specifying a mind amounts to specifying an algebraic closure operator
together with an axiom set, while the rule formalism provides a finite-premise
presentation of that closure structure.

\subsection{Ordered and Unordered Information}\label{sec:ordered}
In classical information theory, the information content of a signal is treated as a property of the source model, independent of the particular receiver. In teaching, however, the usefulness of information is fundamentally relative. The same explanation that substantially reduces uncertainty for a prepared learner may convey little or no usable information to a novice.

This relativity arises because the ability to extract usable information
depends on two internal factors: the learner's prerequisite structure
$\mathcal{E}_\mind$ and the learner's acquired concept set $\mc K$ at the
time of interaction. A concept that is within reach for one mind may be
completely inaccessible to another, either because the two minds operate under
different prerequisite rules, or because they share the same rules but begin
from different acquired concept sets.

Formally, one-step accessibility of a concept is determined by the expansion
map: a concept $c$ is reachable from the current acquired concept set $\mc K$
if and only if $c \in \Phi_\mind(\mc K)$. Consequently, the information
conveyed by a signal is not determined by the signal alone, but by its
position with respect to the learner's mind. This relationship defines the
effective channel through which teaching occurs.

\begin{definition}[Ordered and unordered concept]\label{def:ordered}
Let $\mind$ be a mind and let $\mc K \subseteq \mc C$ be the set of concepts the learner currently knows. A concept $c \in \mc C$ is:
\begin{enumerate}[label=\textnormal{(\roman*)}]
\item \emph{Ordered} for $(\mind,\mc K)$ if $c \in \Phi_\mind(\mc K)$. Equivalently, either $c \in \mc K$, or there exists a rule $(\mc S,c)\in\mathcal E_\mind$ such that $\mc S \subseteq \mc K$.
\item \emph{Unordered} for $(\mind,\mc K)$ if $c \notin \Phi_\mind(\mc K)$. Equivalently, $c \notin \mc K$ and for every rule $(\mc S,c)\in\mathcal E_\mind$, at least one prerequisite in $\mc S$ is missing from $\mc K$.
\end{enumerate}
\end{definition}


{At any given stage of the learning process, the set $\mc K$ represents the
concepts \emph{actually acquired so far}. It is \emph{not} assumed to be closed
under inference. The closure $\cl_\mind(\mc K)$ represents the set of concepts
that are \emph{in principle reachable} from $\mc K$ under the learner's
expansion rules. Accordingly, it need not coincide with the learner's current
acquired set at a given moment. Later, when we model teaching
dynamics (see \Cref{sec:teaching}), the evolving state $\mc K_t$ will represent the concepts the learner
has acquired by time $t$, whereas $\cl_\mind(\mc K_t)$ will describe the
concepts that are potentially accessible from that state.}

\begin{remark}
The distinction between ordered and unordered concepts concerns decodability at the present moment, not whether a signal can be stored for later use. A learner with memory could buffer a signal targeting a currently unordered concept; for example, a student might copy down a formula they do not yet understand. Once the prerequisite concepts enter $\mc K$, the stored signal may become decodable retroactively.
In the memoryless parsing model introduced in \Cref{def:parsing}, by contrast, a signal targeting an unordered concept is lost immediately. A natural extension would replace that parser with a delayed-parsing variant, in which raw signals are buffered and re-parsed whenever $\mc K$ expands. Such a model could lower the teaching-time lower bound, since information presented too early would no longer be wasted.
\end{remark}

\begin{definition}[Valid ordered curriculum]\label{def:valid-curriculum}
Let $\mind$ be a mind and let $\mc K_0 \subseteq \mc C$ be an initial knowledge set. A possibly empty finite sequence $
\gamma=((\mc S_i,c_i))_{i=1}^L$, $L\ge 0$,
is a \emph{valid ordered curriculum} starting from $\mc K_0$ if:
\begin{enumerate}[nosep, label=\textnormal{(\roman*)}]
\item $(\mc S_i,c_i)\in\mathcal E_\mind$ for each $i=1,\dots,L$;
\item defining recursively 
\begin{equation}
\label{eq:recursive-accumulated-knowledge}
\mc K_i=\mc K_{i-1}\cup\{c_i\}, \quad i=1,\dots,L,
\end{equation}
one has $
\mc S_i\subseteq \mc K_{i-1}$ for every $i=1,\dots,L$.
\end{enumerate}
\end{definition}

Definition~\ref{def:valid-curriculum} formalizes the idea that a curriculum must respect prerequisites at every step. The rule $(\mc S_i,c_i)$ can be used only when all concepts in $\mc S_i$ are already contained in the current set $\mc K_{i-1}$. Thus the curriculum follows a prerequisite-respecting path, updating the set of concepts acquired by the learner one step at a time. Here $\mc K_i$ denotes the set of concepts acquired after the first $i$ steps, so that $
\mc K_0 \subseteq \mc K_1 \subseteq \cdots \subseteq \mc K_L$.
\begin{theorem}[Ordering theorem]\label{thm:ordering}
For any mind $\mind$ and any target $c^* \in \mc U_\mind$, there exists a valid ordered curriculum $
\gamma=((\mc S_1,c_1),\ldots,(\mc S_L,c_L))$, $L\ge 0$,
starting from $\Acal_\mind$, such that, if $(\mc K_i)_{i=1,\ldots, L}$ is constructed as in \eqref{eq:recursive-accumulated-knowledge} with $\mc K_0 = \Acal_\mind$,
then $
c^*\in \mc K_L$.
\end{theorem}

\begin{example}[Valid ordered curricula for the two minds]
\label{ex:arithmetic-curricula}
We illustrate \Cref{def:valid-curriculum} using the two minds of
\Cref{ex:arithmetic-minds}, both starting from the initial concept set $
\mc K_0=\{a\}$.

\noindent
\textit{Mind~1 (algorithmic).}
A valid ordered curriculum for Mind~1 is
\[
\gamma_1=(r_1,r_2,r_3),
\qquad
r_1=(\{a\},b),\quad
r_2=(\{b\},c),\quad
r_3=(\{b,c\},d).
\]
Writing $c_1=b$, $c_2=c$, $c_3=d$, and defining
\[
\mc K_0^{(1)}=\{a\},
\qquad
\mc K_i^{(1)}=\mc K_{i-1}^{(1)}\cup\{c_i\}
\quad (i=1,2,3),
\]
we obtain
\[
\mc K_1^{(1)}=\{a,b\},\qquad
\mc K_2^{(1)}=\{a,b,c\},\qquad
\mc K_3^{(1)}=\{a,b,c,d\}.
\]
Indeed, at each step the prerequisite set of the selected rule is contained in
the current acquired concept set.

\medskip
\noindent
\textit{Mind~2 (visual).}
A valid ordered curriculum for Mind~2 is
\[
\gamma_2=(r_1',r_2',r_3'),
\qquad
r_1'=(\{a\},c),\quad
r_2'=(\{c\},b),\quad
r_3'=(\{b,c\},d).
\]
Writing $c_1'=c$, $c_2'=b$, $c_3'=d$, and defining
\[
\mc K_0^{(2)}=\{a\},
\qquad
\mc K_i^{(2)}=\mc K_{i-1}^{(2)}\cup\{c_i'\}
\quad (i=1,2,3),
\]
we obtain
\[
\mc K_1^{(2)}=\{a,c\},\qquad
\mc K_2^{(2)}=\{a,b,c\},\qquad
\mc K_3^{(2)}=\{a,b,c,d\}.
\]
Again, each rule is applicable when used.

\medskip
\noindent
Thus the two minds admit different valid ordered curricula from the same
starting set. In particular, their first steps must differ. For Mind~1 the only
rule whose prerequisite set is contained in $\{a\}$ is $(\{a\},b)$, whereas for
Mind~2 the only such rule is $(\{a\},c)$. This suggests that a single common
curriculum cannot in general respect the structural requirements of both minds
simultaneously, foreshadowing the impossibility result of
\Cref{sec:impossibility}.
\end{example}
{\begin{proposition}[Curricula stay inside understanding horizon]\label{prop:unteachable}
Let $\mind$ be a mind, and let $
\gamma=((\mc S_i,c_i))_{i=1}^L$
be a valid ordered curriculum starting from $\Acal_\mind$. Let  $(\mc K_i)_{i=1, \ldots, L}$ be constructed as in \eqref{eq:recursive-accumulated-knowledge} with $\mc K_0 = \Acal_\mind$. Then $
\mc K_i \subseteq \mc U_\mind$ for every $i=0,1,\dots,L$.
In particular, if $c^* \notin \mc U_\mind$, then no valid ordered curriculum starting from $\Acal_\mind$ can reach $c^*$.
\end{proposition}

\Cref{prop:unteachable} draws a boundary around what any curriculum can achieve. If a concept does not belong to the understanding horizon $\mc U_\mind$, then no sequence of rule applications, however long or carefully arranged, can produce it. The barrier is structural, not epistemic: it is not that the teacher lacks information or that the curriculum is poorly designed, but that the expansion rules of the learner do not connect the axioms to the target concept. In this sense, the understanding horizon $\mc U_\mind$ is the theoretical horizon of the mind $\mind$.

A concrete illustration is the attempt to convey the visual experience of the color purple to a learner who has been blind from birth. Here the target concept is not the word purple or its descriptive use, but the sensory concept associated with its visual appearance. Such a learner may understand many relational facts about color: that purple is classified between blue and red in certain systems, that particular objects are called purple by sighted speakers, or that light associated with purple occupies a certain range of wavelengths. But if the learner's mind contains no rule path from its existing concepts to that sensory target, then no curriculum, however long or ingeniously ordered, can reach it.
\subsection{Reachable acquired concept sets}\label{sec:reachable-states}

The closure operator $\cl_\mind$ identifies what is reachable in principle from
the axiom set, but it does not describe the intermediate concept sets through
which a learner may pass on the way to that horizon. This distinction is
structurally important and closely related to a central idea in the literature
on knowledge spaces, where one studies not only which
concepts are ultimately attainable, but also which intermediate learning states
are feasible along a learning process \citep{ref:doignon1999knowledge,ref:korte1983structural}.
Closure is a global notion: if a concept lies in $\cl_\mind(\mc K)$, then it is
eventually reachable from $\mc K$, but it need not already belong to the
current acquired concept set $\mc K$. In particular, closure alone does not
record which subsets of $\mc U_\mind$ can arise by successive
prerequisite-respecting acquisitions, one concept at a time.

For the structural theory of teaching, we therefore need a finer object than
the understanding horizon alone. We introduce the family of \emph{reachable
acquired concept sets}: those subsets of $\mc U_\mind$ that can be built from
the axioms by a finite sequence of locally valid acquisitions. This family is
the natural state space for teaching dynamics. 
Later we show that, under a
finite-horizon assumption, it has the combinatorial structure familiar
from the knowledge-space literature: after shifting by the axiom core, it forms
an \textit{antimatroid}, equivalently, a learning space. Thus the framework does not take
feasible learning states of \citep{ref:doignon1999knowledge} as primitive; rather, it derives them from axioms and
expansion rules of a mind.

\begin{definition}[Reachable acquired concept sets]\label{def:reachable}
A set $\mc K\subseteq \mc U_\mind$ is \emph{reachable} if there exists a finite chain $
\Acal_\mind=\mc K_0\subset \mc K_1\subset \cdots\subset \mc K_L=\mc K$
such that for each $i=0,\dots,L-1$, $
\mc K_{i+1}=\mc K_i\cup\{c_i\}$,
$c_i\in \Phi_\mind(\mc K_i)\setminus \mc K_i$.
Any such chain is called a \emph{witnessing chain} for the reachability of $\mc K$.
\end{definition}

\begin{assumption}[Finite understanding horizon]\label{ass:finite}
The understanding horizon $\mc U_\mind=\cl_\mind(\Acal_\mind)$ is finite.
\end{assumption}
\Cref{ass:finite} is imposed only to place the reachable family within the finite
combinatorial framework of learning spaces and antimatroids. The definition of
reachability itself does not require finiteness.
Under this assumption, we define the \emph{reachable family} of mind $\mind$ as $
\mathbb K_\mind
=
\{\mc K\subseteq \mc U_\mind : \mc K \text{ is reachable from } \Acal_\mind\}$.

The reachable family $\mathbb K_\mind$ will later serve as the state space for
the teaching dynamics in \Cref{sec:teaching}, so its internal structure is
central to the theory.
The next
proposition shows that this family has three basic features. It has a
distinguished minimum state, every non-minimal reachable state can be obtained
from another reachable state by adding a single concept, and it is closed under
unions. These properties are natural from the perspective of learning: one can
build feasible states step by step, and compatible partial acquisitions can be
combined. They also place the reachable family in close correspondence with the
combinatorial objects studied in the literature on learning spaces \citep{ref:doignon1999knowledge} and
antimatroids \citep{ref:korte1983structural}.

\begin{proposition}[Structure of the reachable family]\label{prop:reach-structure}
Under \Cref{ass:finite}, the family $\mathbb K_\mind$ is finite and satisfies:
\begin{enumerate}[nosep,label=\textnormal{(\roman*)}]
\item $\Acal_\mind$ is the minimum element of the partially ordered set $(\mathbb K_\mind,\subseteq)$;
\item for every $\mc K\in\mathbb K_\mind$ with $\mc K\neq \Acal_\mind$, there exists $\mc K'\in\mathbb K_\mind$ such that $
\mc K'\subset \mc K$ and
$|\mc K\setminus \mc K'|=1$;
\item if $\mc K,\mc K'\in\mathbb K_\mind$, then $\mc K\cup\mc K'\in\mathbb K_\mind$;
\item $\mc U_\mind$ is the maximum element of $(\mathbb K_\mind,\subseteq)$;
\item ordered by inclusion, $(\mathbb K_\mind,\subseteq)$ is a finite join-semilattice, and for every $\mc K,\mc K'\in\mathbb K_\mind$ the join is given by $
\mc K\vee \mc K'=\mc K\cup \mc K'$.
\end{enumerate}
\end{proposition}
Properties \textnormal{(i)} through \textnormal{(iii)} identify the core
combinatorial features of the reachable family: a distinguished minimum state,
one-step accessibility, and union-closure. These are precisely the ingredients
that connect the reachable family to the notions of learning space~\citep{ref:doignon1999knowledge} and
antimatroid~\citep{ref:korte1983structural}.

To make the connection precise, we recall both
concepts and their equivalence.
{An antimatroid on a finite set $\mathcal{E}$ is a family $\mathcal{F} \subseteq 2^{\mathcal E}$ satisfying: (i) $\varnothing \in \mathcal{F}$; (ii) for every nonempty $\mathcal{S} \in \mathcal{F}$, there exists $x \in \mathcal{S}$ such that $\mathcal{S} \backslash\{x\} \in \mathcal{F}$ (accessibility); and (iii) $\mathcal{F}$ is union-closed. While matroids \citep{ref:hassler} axiomatize independence structures where feasibility is closed downward: every subset of a feasible set is feasible; antimatroids \citep{ref:korte1983structural} capture the complementary pattern: feasibility is closed upward under unions, modeling sequential construction under precedence constraints. Independently, \citep{ref:doignon1999knowledge} arrived at the same mathematical structure from a different motivation: modeling the feasible knowledge states of a human learner. They called the resulting object a learning space \cite[Theorem 7]{ref:doignon2015knowledge}, \citep{ref:doignon-falmagne-2016-learning-spaces}, which is an antimatroid.}
The standard definition of a learning space takes the empty set as the minimum
element, modeling a learner who begins with no knowledge. In our setting the
learner starts from the axiom set $\Acal_\mind$, so we introduce a shifted
variant that replaces $\varnothing$ with $\Acal$.
\begin{definition}[$\Acal$-based learning space]\label{def:A-based-learning-space}
Let $\mc U$ be a finite set and let $\Acal\subseteq \mc U$. A family $
\mathbb F\subseteq 2^{\mc U}$
is called an \emph{$\Acal$-based learning space} if:
\begin{enumerate}[nosep,label=\textnormal{(\roman*)}]
\item $\Acal\in\mathbb F$ and every $\mc K\in\mathbb F$ satisfies $\Acal\subseteq \mc K$;
\item for every $\mc K\in\mathbb F$ with $\mc K\neq \Acal$, there exists $x\in \mc K\setminus \Acal$ such that $\mc K\setminus\{x\}\in\mathbb F$;
\item $\mathbb F$ is union-closed.
\end{enumerate}
Equivalently, the shifted family $
\widehat{\mathbb F}
=
\{\mc K\setminus \Acal : \mc K\in\mathbb F\}
\subseteq 2^{\mc U\setminus \Acal}$
is an antimatroid.
\end{definition}

\begin{corollary}[Shifted antimatroid structure]\label{cor:reachable-learning-space}
Under \Cref{ass:finite}, the reachable family $\mathbb K_\mind$ is an $\Acal_\mind$-based learning space. Equivalently, the shifted family $ 
\widehat{\mathbb K}_\mind
=
\{\mc K\setminus \Acal_\mind : \mc K\in\mathbb K_\mind\}
\subseteq 2^{\mc U_\mind\setminus \Acal_\mind}$
is an antimatroid.
\end{corollary}
The next theorem characterizes the reachable families generated by minds: they are precisely the $\Acal$-based learning spaces.
\begin{theorem}[Representation of reachable families]\label{thm:representation-learning-space}
Let $\mc C$ be a finite set, let $\Acal\subseteq \mc C$, and let $\mathbb F\subseteq 2^{\mc C}$. The following are equivalent:
\begin{enumerate}[nosep,label=\textnormal{(\roman*)}]
\item $\mathbb F$ is an $\Acal$-based learning space;
\item there exists a mind $
\mind=(\mc C,\Acal,\mathcal E_\mind)$
whose reachable family satisfies $
\mathbb K_\mind=\mathbb F$.
\end{enumerate}
Moreover, when \textnormal{(i)} holds, the mind
$\mind_{\mathbb F}=(\mc C,\Acal,\mathcal E_{\mathbb F})$ with rule set $
\mathcal E_{\mathbb F}
=
\{(\mc S,c): \mc S\in \mathbb F,\ c\in \mc C\setminus \mc S,\ 
\mc S\cup\{c\}\in \mathbb F\}$
satisfies $\mathbb K_{\mind_{\mathbb F}}=\mathbb F$.
\end{theorem}

\Cref{fig:antimatroid-hasse} illustrates the reachable family for a mind with axiom set $\Acal_\mind=\{a\}$ and expansion rules $
\{a\}\Rightarrow b,~
\{a\}\Rightarrow c,~
\{b,c\}\Rightarrow d$.
Starting from $\underline{\mc K}=\{a\}$, the learner can acquire $b$ or $c$ in either order, since both are individually unlocked by the axiom. However, $d$ becomes reachable only once both $b$ and $c$ have been acquired, so $\{a,b,c\}$ is the unique gateway to $d$. The set $\{a,b,d\}$, for instance, does not belong to $\mathbb K_\mind$ because the rule for $d$ requires $c$, which is absent. The figure makes both accessibility and union-closure visible.

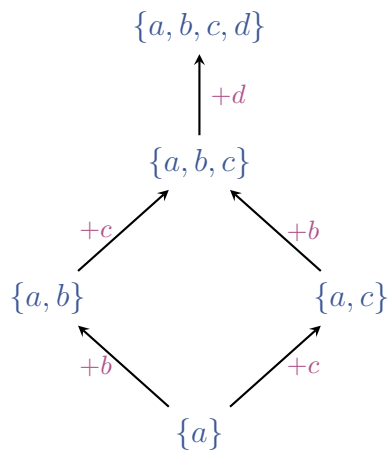
\begin{wrapfigure}{r}{0.39\textwidth}
\centering
\vspace{-0.7cm}
\begin{tikzpicture}[node distance=1.5cm and 2cm, thick, >=stealth]
    \node (a) at (0,0) {\color{mblue}$\{a\}$};
    \node (ab) at (-2,1.8) {\color{mblue}$\{a,b\}$};
    \node (ac) at (2,1.8) {\color{mblue}$\{a,c\}$};
    \node (abc) at (0,3.6) {\color{mblue}$\{a,b,c\}$};
    \node (abcd) at (0,5.4) {\color{mblue}$\{a,b,c,d\}$};

    \draw[->] (a) -- (ab) node[midway,left] {\footnotesize \color{spring}$+b$};
    \draw[->] (a) -- (ac) node[midway,right] {\footnotesize \color{spring}$+c$};
    \draw[->] (ab) -- (abc) node[midway,left] {\footnotesize \color{spring}$+c$};
    \draw[->] (ac) -- (abc) node[midway,right] {\footnotesize \color{spring}$+b$};
    \draw[->] (abc) -- (abcd) node[midway,right] {\footnotesize \color{spring}$+d$};
\end{tikzpicture}
\vspace{-0.5cm}
\caption{\small The reachable family $\mathbb K_\mind$ for a mind with axiom set $\Acal_\mind=\{a\}$ and expansion rules $\{a\}\Rightarrow b$, $\{a\}\Rightarrow c$, $\{b,c\}\Rightarrow d$. The concept $d$ becomes reachable only at $\{a,b,c\}$, where both prerequisites are present. Sets such as $\{a,b,d\}$ are structurally unreachable.}
\label{fig:antimatroid-hasse}
\end{wrapfigure}
\Cref{thm:representation-learning-space} characterizes the reachable families generated by minds as the $\Acal$-based learning spaces.
This has two consequences for the present work. First, the feasible
knowledge states of a mind need not be postulated as a primitive; they are
derived from axioms and expansion rules, and the resulting state space
automatically inherits the rich combinatorial structure of an antimatroid.
Second, the converse direction guarantees that the mind formalism is fully
expressive: any learning space one might wish to study can be generated by a
suitably chosen mind. Thus the structural and the generative viewpoints are
equivalent. We note, however, that not every union-closed family above the axioms
qualifies as an $\Acal$-based learning space. Accessibility is an additional
requirement. It rules out degenerate state spaces in which the learner cannot
progress one concept at a time; see \Cref{cor:not-every-knowledge-space}.

\section{Teaching and Learning Dynamics}\label{sec:teaching}\label{sec:arrow}

{Understanding characterizes which concepts are in principle accessible under a
prerequisite structure. Teaching introduces a second challenge beyond
accessibility: the learner must identify the teaching target. A signal about
addition in a mathematics course, for example, may indicate that addition is
itself the intended endpoint, or it may be an intermediate step on the way to
multiplication. This is the identification component of teaching.

It is here that intentionality enters. A teaching move is not merely the
presentation of a concept; it is an action chosen in light of a target and
interpreted by the learner as evidence about that target. To represent this
\textit{as}ymmetry, we model the target concept as a latent variable known to the teacher
and unknown to the learner, and we represent the learner's evolving belief as a
probability distribution over candidate targets.

The latent target need not be interpreted only as the teacher's intended
endpoint. It may also be read as the higher-level concept that renders the
currently acquired material globally coherent. On this interpretation, learning
involves two coupled dimensions: the acquired concept set expands, while the
learner simultaneously infers which larger target those concepts are organizing
toward. A concept may therefore be acquired locally before its place in the
larger conceptual graph is understood. For example, a learner may acquire many
concepts from electromagnetism and electronics while still lacking the bridge
concept that connects them to wireless communication. Once that target is
identified, previously disconnected material becomes integrated as part of a
single explanatory structure.

Teaching dynamics therefore involve both structural and epistemic progress.
Structural progress is governed by the prerequisite structure: once the learner
is at an acquired set from which a concept is \emph{ordered}, and the
appropriate signal is successfully parsed, that concept enters the learner's
acquired concept set. Epistemic progress, by contrast, concerns the gradual
resolution of uncertainty about the target. Because the learner does not know
which target the teacher intends, each signal must play a dual role: it must be
a valid instructional step in the prerequisite structure, and it must
simultaneously provide evidence that distinguishes the intended target from the
alternatives. From the learner's perspective, the observed signal is therefore a
random variable whose distribution depends on both the unknown target and the
teaching strategy. Each round can convey only a bounded amount of usable
information about the latent target, and the total teaching time is governed by
the rate at which this epistemic uncertainty is resolved.
If the learner knew the target from the outset, the epistemic dimension would
disappear and teaching would reduce to the purely structural problem of
reaching a known target by a valid curriculum.}

We now make these ideas concrete by introducing a stochastic model of teaching.

\subsection{A Stochastic Model of Teaching}
\label{sec:stoc-teaching}

Fix a probability space $(\Omega_0,\Fcal,\mathbb P)$ on which all random variables below are defined. Let $\Omega \subseteq \mc C$ be a finite set of \emph{target concepts}. Let $
\Theta:\Omega_0\to\Omega$
be an $\Omega$-valued random variable representing the realized target concept, known to the teacher but unknown to the learner. The learner's goal is to identify $\Theta$.

Let $\mc Z$ be a finite set, called the \emph{teaching signal set}, consisting of the raw signals the teacher can emit. Let $\bot\notin\mc Z$ be an additional symbol representing a null observation
produced when a signal cannot be parsed at the learner's current knowledge
state. The learner observation set is $
\mc Y=\mc Z\cup\{\bot\}$.

\begin{definition}[Signal target map]
\label{def:signal-target-map}
A \emph{signal target map} is a function $
\tgt:\mc Z\to\mc C$
that assigns to each raw teaching signal $z\in\mc Z$ the concept
$\tgt(z)\in\mc C$ that the signal is intended to teach.
We assume that every target concept is associated with at least one raw
signal, that is, $
\Omega \subseteq \operatorname{im}(\tgt)$.
\end{definition}
For each concept $c\in\mc C$, the fiber $
\tgt^{-1}(c)=\{z\in\mc Z:\tgt(z)=c\}$
is the set of all raw signals designed to teach $c$, representing different
explanations, examples, or phrasings of the same concept. Signals in the fiber $\tgt^{-1}(c)$ all target the same concept and therefore
have the same structural effect on the learner's acquired concept set.
However, they may still differ informationally: distinct signals in the fiber
can encode different information about the latent target $\Theta$.

\begin{remark}[Fixed signal system and notation]\label{rmk:fixed-signal-system}
The raw signal alphabet $\mc Z$ and the target map $\tgt$ are treated as fixed
throughout a given teaching problem. Capacity quantities introduced later
therefore depend not only on the mind $\mind$ and the acquired concept set
$\mc K$, but also on this signal system $(\mc Z,\tgt)$. When no ambiguity
arises, we suppress this dependence in the notation.
\end{remark}
The signal target map $\tgt$ and the latent target $\Theta$ play complementary
but distinct roles. The random variable $\Theta\in\Omega$ specifies what the
learner must ultimately identify: the realized target concept. The map $\tgt$
specifies what each individual signal is about: a signal $z$ with $\tgt(z)=c$
is designed to teach concept $c$, which may or may not equal $\Theta$.

In general, signals targeting prerequisite concepts may need to be presented
before signals targeting $\Theta$ itself can become usable to the learner.
Thus the teacher's eventual strategy has two degrees of freedom: which concept
to target, and which particular encoding of that concept to use within the
fiber $\tgt^{-1}(c)$. Consequently, a signal may carry information about the
target even when it does not directly target the concept $\Theta$.

We now introduce the \emph{parsing map} $\rho_\mind$, which takes a raw teaching
signal together with the learner's current knowledge set and either passes the
signal through, when the prerequisites are satisfied, or collapses it to the
null token $\bot$ otherwise.

\begin{definition}[Parsing map]\label{def:parsing}
A mind $\mind$ is equipped with a \emph{parsing map} $
\rho_\mind:\mc Z\times 2^{\mc C}\to \mc Z\cup\{\bot\}$,
where $\bot$ is a \emph{null token} indicating that the signal is unparseable. For a signal $z\in\mc Z$ with target $c=\tgt(z)$ and a knowledge set $\mc K\subseteq\mc C$:
\begin{enumerate}[nosep, label=\textnormal{(\roman*)}]
\item $
\rho_\mind(z,\mc K)=z$
if $c\in \Phi_\mind(\mc K)$, equivalently, if either $c\in\mc K$ already or there exists a rule $(\mc S,c)\in\mathcal E_\mind$ with $\mc S\subseteq \mc K$;
\item $
\rho_\mind(z,\mc K)=\bot$
if $c\notin \Phi_\mind(\mc K)$, equivalently, if $c\notin\mc K$ and no rule for $c$ has all its prerequisites in $\mc K$.
\end{enumerate}
\end{definition}
The condition $c\in\Phi_\mind(\mc K)$ is the ordered condition of \Cref{def:ordered}. A concept may have multiple prerequisite sets, and the signal is parseable if any one of them is satisfied.

\paragraph{Dynamics.}
We model teaching as a repeated interaction between a teacher and a learner unfolding over discrete rounds $t=0,1,2,\ldots$. 
{The model uses a concept-level time scale: one round represents a single
instructional interaction in which the teacher emits one raw signal, the
learner observes its parsed version, and the learner's acquired concept set may
be updated as a result.

We take the learner's initial acquired concept set to be the axiom set of the
mind:} $
\mc K_0=\Acal_\mind$.
For each $t\ge 0$, the set $\mc K_t\subseteq\mc C$ denotes the concepts acquired by the learner after the first $t$ rounds of instruction.
At round $t+1$, the teacher emits a raw signal $
Z_{t+1}\in\mc Z$.
Given the learner's current acquired concept set $\mc K_t$, the learner observation is the parsed signal $
Y_{t+1}=\rho_\mind(Z_{t+1},\mc K_t)\in\mc Y$.

\begin{definition}[Concept-acquisition update rule]\label{def:knowledge-update}
Given the parsed observation $Y_{t+1}\in\mc Y$, the learner's acquired concept set evolves according to
\[
\mc K_{t+1}=
\begin{cases}
\mc K_t\cup\{\tgt(Y_{t+1})\} & \text{if } Y_{t+1}\in \mc Z,\\[2mm]
\mc K_t & \text{if } Y_{t+1}=\bot.
\end{cases}
\]
\end{definition}
Under this update rule, each round can add at most one newly acquired concept, namely the concept targeted by the parsed signal when parsing succeeds. Time is therefore measured in units of concept-level teaching opportunities.

The rule in \Cref{def:knowledge-update} has two immediate consequences. First, acquisition is monotone: $
\mc K_t\subseteq \mc K_{t+1}$ for all $t\ge 0$.
Second, the set $\mc K_t$ records only concepts that have been explicitly acquired through parsed instruction; it is not automatically closed under the expansion rules. Thus a concept may already be reachable from $\mc K_t$, in the sense that $c\in\Phi_\mind(\mc K_t)$, without yet belonging to $\mc K_t$ itself. The learner acquires such a concept only at a later round in which it receives a parseable signal targeting $c$. Therefore, $\mc U_\mind=\cl_\mind(\Acal_\mind)$ describes what is in principle reachable from the axioms, whereas the process $(\mc K_t)_{t\ge 0}$ describes what has actually been acquired over time.
{\begin{lemma}[The instructional process stays inside the reachable family]
\label{lem:Kt-reachable}
For every $t\ge 0$, one has $
\mc K_t\in \mathbb K_\mind$ \text{almost surely.}
\end{lemma}
\Cref{lem:Kt-reachable} shows that the stochastic teaching process evolves
within the reachable family $\mathbb K_\mind$. Thus the family introduced in
\Cref{sec:reachable-states} not only describes structurally feasible knowledge
states but also forms the natural state space for the instructional dynamics.

\begin{definition}[Admissible teaching strategy]\label{def:admissible-strategy}
An \emph{admissible teaching strategy} is a sequence of stochastic kernels
\[
\kappa_{t+1}(\,\cdot \mid \theta,y_1,\dots,y_t)\in \Delta(\mc Z),
\qquad t\ge 0,
\]
so that, conditional on the realized target $\Theta=\theta$ and the parsed
history $(Y_1,\dots,Y_t)=(y_1,\dots,y_t)$, the teacher chooses the next raw
signal $Z_{t+1}$ according to $\kappa_{t+1}$.
\end{definition}
Because the learner's epistemic objective is to identify the latent target
$\Theta$, it maintains at each time $t$ a belief over the possible target
concepts. This belief is updated from the parsed observations
$Y_1,\dots,Y_t$, rather than from the raw teacher emissions
$Z_1,\dots,Z_t$, which are not directly observed by the learner.
Accordingly, define the learner's information filtration by
\[
\Fcal_t=\sigma(Y_1,\dots,Y_t)\subseteq \Fcal,\qquad t\ge 1,
\]
and set $
\Fcal_0=\{\varnothing,\Omega_0\}$.
Given a fixed prior $\pi_0$ and a fixed admissible teaching strategy, the
learner's posterior at time $t$ is the random probability vector
$\pi_t\in\Delta(\Omega)$ defined by
\[
\pi_t(c)=\mathbb P(\Theta=c\mid \Fcal_t),
\qquad c\in\Omega.
\]
The conditional probability is taken with respect to the probability law
induced by the prior $\pi_0$ and the admissible teaching strategy. Thus the
learner is modeled as Bayesian: its belief state at time $t$ is the posterior
distribution of the latent target given the parsed observation history.
\begin{definition}[Learning state]\label{def:learning-state}
A \emph{learning state} at time $t$ is a pair $(\mc K_t,\pi_t)$ where
\begin{enumerate}[nosep,label=(\roman*)]
\item $\mc K_t\subseteq\mc C$ is the learner's acquired concept set at time $t$;
\item $\pi_t\in\Delta(\Omega)$ is the learner's posterior belief over target concepts.
\end{enumerate}
\end{definition}
Thus the learning state records both dimensions of progress in the teaching
process: structural progress, captured by the acquired concept set $\mc K_t$,
and epistemic progress, captured by the posterior belief $\pi_t$ about the
latent target. The stochastic teaching dynamics therefore evolve on the
product space $\mathbb K_\mind \times \Delta(\Omega)$.
\begin{definition}[Completion]\label{def:completion}
Teaching is \emph{complete} at time $\tau$ if both
\begin{enumerate}[nosep,label=(\roman*)]
\item \emph{target acquisition}: $\Theta\in\mc K_\tau$;
\item \emph{identification}: $\pi_\tau(\Theta)=1$.
\end{enumerate}
\end{definition}
The framework therefore distinguishes three related notions. First, the
understanding horizon $
\mc U_\mind=\cl_\mind(\Acal_\mind)$
is the set of concepts that are in principle reachable from the axioms under
the expansion rules of the mind. Second, the time-indexed set $\mc K_t$
records which concepts have actually been acquired through instruction by time
$t$. Third, the completion condition of \Cref{def:completion} formalizes
successful teaching of the target: it requires both acquisition of the target,
$\Theta\in\mc K_\tau$, and identification of the target,
$\pi_\tau(\Theta)=1$.
Acquisition without identification corresponds to having acquired a concept
without yet knowing that it is the intended target. Identification without
acquisition corresponds to knowing which concept is intended without yet having
reached it. Completion requires both.

\begin{example}[A full teaching interaction]\label{ex:full-teaching}
Let $\mc C=\{a,b,c,d\}$ with the informal readings
\[
a=\text{counting},\qquad
b=\text{addition},\qquad
c=\text{arrays},\qquad
d=\text{multiplication}.
\]
Fix Mind~1 from \Cref{ex:arithmetic-minds}, with axiom set $\Acal_\mind=\{a\}$ and expansion rules
\[
\{a\}\Rightarrow b,\qquad
\{b\}\Rightarrow c,\qquad
\{b,c\}\Rightarrow d.
\]
Let $\Omega=\{b,c,d\}$, let the prior be uniform on $\Omega$, and let the teacher use the deterministic policy
\[
\Theta=b:\ (Z_1,Z_2,Z_3)=(z_b^{(1)},z_b^{(1)},z_b^{(1)}),
\]
\[
\Theta=c:\ (Z_1,Z_2,Z_3)=(z_b^{(1)},z_c^{(1)},z_c^{(1)}),
\]
\[
\Theta=d:\ (Z_1,Z_2,Z_3)=(z_b^{(1)},z_c^{(1)},z_d^{(1)}),
\]
where $\tgt(z_b^{(1)})=b$, $\tgt(z_c^{(1)})=c$, and $\tgt(z_d^{(1)})=d$.
Suppose the realized target is $\Theta=d$. The learner starts from
\[
\mc K_0=\{a\},
\qquad
\pi_0(b)=\pi_0(c)=\pi_0(d)=\frac13.
\]

At $t=0$, the teacher emits $Z_1=z_b^{(1)}$. Since $
b\in\Phi_\mind(\{a\})$,
the signal is parseable, so
\[
Y_1=z_b^{(1)}
\qquad\text{and}\qquad
\mc K_1=\{a,b\}.
\]
Because the same first signal is prescribed under all three targets, the observation $Y_1=z_b^{(1)}$ does not yet distinguish among them, and therefore $
\pi_1=\pi_0$.

At $t=1$, the teacher emits $Z_2=z_c^{(1)}$. Since $b$ has already been acquired, the concept $c$ is now ordered, so the signal is parseable:
\[
Y_2=z_c^{(1)}
\qquad\text{and}\qquad
\mc K_2=\{a,b,c\}.
\]
Under the stated policy, the history $(Y_1,Y_2)=(z_b^{(1)},z_c^{(1)})$ is inconsistent with $\Theta=b$. Hence the posterior assigns zero mass to $b$ and splits mass equally between $c$ and $d$:
\[
\pi_2(b)=0,
\qquad
\pi_2(c)=\pi_2(d)=\frac12.
\]

At $t=2$, the teacher emits $Z_3=z_d^{(1)}$. Since both $b$ and $c$ are now present, the concept $d$ is ordered, so the signal is parseable:
\[
Y_3=z_d^{(1)}
\qquad\text{and}\qquad
\mc K_3=\{a,b,c,d\}.
\]
Now the full observation history is consistent only with $\Theta=d$, so $
\pi_3=\delta_d$.

Thus teaching is complete at time $\tau=3$: the learner has both acquired the target, $
\Theta=d\in\mc K_3$,
and identified it, $
\pi_3(d)=1$.
This example illustrates the distinction between structural acquisition, encoded by the process $(\mc K_t)$, and epistemic identification, encoded by the posterior process $(\pi_t)$.
\end{example}

\subsection{The Epistemic Arrow of Time}
\label{sec:epistemic-arrow}
We now formalize the epistemic component of the teaching dynamics. The key
question is how the learner's uncertainty about the latent target evolves as
parsed observations accumulate over time. This motivates the term
\emph{epistemic arrow of time}: although particular observations may be
uninformative, posterior uncertainty can only decrease in conditional
expectation under Bayesian updating. The information-theoretic notions used
below are standard; see, for example, \cite[\S 2]{cover}.

\paragraph{Information-theoretic quantities.}\label{def:epistemic-entropy}
Let $X$ and $Y$ be discrete random variables on a probability space
$(\Omega_0,\Fcal,\mathbb P)$ taking values in finite or countable sets
$\mathcal X$ and $\mathcal Y$. We adopt the convention $0\log 0=0$. The
Shannon entropy of $X$ is
\[
\mathds H(X)=-\sum_{x\in\mathcal X}\mathbb P(X=x)\log \mathbb P(X=x).
\]
For a sub-$\sigma$-field $\mc G\subseteq\Fcal$, define the \emph{pathwise
conditional entropy} of $X$ given $\mc G$ by
\[
\mathds H(X\mid \mc G)
=-\sum_{x\in\mathcal X}\mathbb P(X=x\mid \mc G)\log \mathbb P(X=x\mid \mc G).
\]
Its expectation $\E[\mathds H(X\mid \mc G)]$ is the usual conditional entropy.
For brevity, we write
\[
\mathds H(X\mid Y)=\E[\mathds H(X\mid \sigma(Y))].
\]
The mutual information between $X$ and $Y$ is
\[
\mathds I(X;Y)=\mathds H(X)-\mathds H(X\mid Y).
\]
The conditional mutual information given $\mc G$ is
\[
\mathds I(X;Y\mid \mc G)
=\mathds H(X\mid \mc G)
-\E[\mathds H(X\mid \mc G\vee \sigma(Y))\mid \mc G].
\]

In the teaching model, $\Theta$ is $\Omega$-valued and the learner filtration
is
\[
\Fcal_t=\sigma(Y_1,\dots,Y_t).
\]
We define the \emph{epistemic entropy} at time $t$ by
\[
H_t=\mathds H(\Theta\mid \Fcal_t).
\]
Since $\pi_t(c)=\mathbb P(\Theta=c\mid \Fcal_t)$, this may be written as
\[
H_t=-\sum_{c\in\Omega}\pi_t(c)\log \pi_t(c)
\qquad\text{a.s.}
\]
Thus $H_t$ is the Shannon entropy of the learner posterior at time
$t$.

\begin{proposition}[Entropy drop equals information flow]
\label{prop:first-law}
The one-round expected reduction in epistemic entropy satisfies $
\mathbb E[H_t-H_{t+1}\mid \Fcal_t]
=\mathds I(\Theta;Y_{t+1}\mid \Fcal_t)$.
\end{proposition}
\Cref{prop:first-law} expresses a conservation principle: the expected
reduction in posterior uncertainty about the target is equal to the
conditional mutual information conveyed by the next parsed observation. In
other words, expected learning progress in one round is precisely the amount
of information that $Y_{t+1}$ carries about the target $\Theta$.

\begin{theorem}[Epistemic arrow of time]\label{thm:arrow}
The epistemic entropy process $(H_t)_{t\ge 0}$ is a supermartingale:
\[
\mathbb E[H_{t+1}\mid \Fcal_t]\le H_t,
\]
with equality if and only if $Y_{t+1}$ is independent of $\Theta$ given
$\Fcal_t$.
\end{theorem}
\Cref{thm:arrow} formalizes the epistemic arrow of time:
posterior uncertainty decreases in conditional expectation, although along
particular sample paths it may increase after a realized observation. Equality
holds when the next observation carries no information about the
target.

\begin{remark}[Bayesian modeling choice]\label{rmk:bayesian}
By defining $\pi_t(c)=\mathbb P(\Theta=c\mid \Fcal_t)$, we have adopted a
Bayesian learner model: the learner belief is the true conditional
distribution of the target given the parsed observation history. This is not
the only possible choice, but it is natural here for three reasons. First,
$\pi_t$ uses all information contained in the observations and nothing else, so
it is determined entirely by the prior $\pi_0$ and the filtration $\Fcal_t$.
Second, because $\pi_t$ is the conditional distribution of $\Theta$ given
$\Fcal_t$, the epistemic entropy $H_t$ coincides with the conditional
entropy $\mathds H(\Theta\mid \Fcal_t)$. This makes mutual information the
natural measure of learning progress: each new observation reduces posterior
uncertainty by precisely $\mathds I(\Theta;Y_{t+1}\mid \Fcal_t)$ in
conditional expectation. Third, the completion condition
$\pi_\tau(\Theta)=1$ then has a strong interpretation: the parsed observations
identify the target, rather than the learner merely arriving at the correct
answer by chance.
\end{remark}

\subsection{Prerequisites and the Relativity of Randomness}
\label{sec:prereq-rel-randomness}

The epistemic arrow of time in \Cref{thm:arrow} describes how uncertainty
evolves once observations are received. It does not, however, determine what
the learner actually observes. In the teaching model the learner does not
observe the raw teacher signal $Z_{t+1}$ directly; instead it receives the
parsed observation $
Y_{t+1}=\rho_\mind(Z_{t+1},\mc K_t)$,
where the parsing map depends on the learner's current acquired concept set.
When the targeted concept is ordered, the signal passes through unchanged; when
prerequisites are missing, the parser collapses the signal to the null token
$\bot$. The effective information channel from $\Theta$ to the learner is
therefore state dependent. In particular, the same raw broadcast may transmit
usable information to one learner while being erased for another.

The next result formalizes this phenomenon. As throughout, conditional
mutual-information expressions given $U_{t+1}$ or $U_{t+1}^c$ are understood on
the event where the relevant conditioning probability is positive, and are
taken to be $0$ otherwise.

\begin{theorem}[Relativity of randomness]\label{thm:relativity}
Let $
C_{t+1}=\tgt(Z_{t+1})$
be the targeted concept, and define the unparseability event $
U_{t+1}=\{C_{t+1}\notin \Phi_\mind(\mc K_t)\}$.
Assume that on parseable rounds the raw teacher signal is informative about the
latent target: $
\mathds I(\Theta;Z_{t+1}\mid \Fcal_t,U_{t+1}^c)>0$.
Then the learner's per-round information transfer exhibits an eventwise
dichotomy:
\begin{align*}
\mathds I(\Theta;Y_{t+1}\mid \Fcal_t,U_{t+1}) &= 0
\qquad \text{\emph{(erasure)}}, \\
\mathds I(\Theta;Y_{t+1}\mid \Fcal_t,U_{t+1}^c) &> 0
\qquad \text{\emph{(informative)}}.
\end{align*}
\end{theorem}
\Cref{thm:relativity} shows that the usable information in a teaching signal is
state dependent. Under the parsing map $\rho_\mind$, if the targeted concept is
unordered at $\mc K_t$, then the parsed observation collapses to $\bot$; by
\Cref{thm:relativity}, the learner receives no further within-event discrimination from the raw signal on that event: conditional on unparseability, the parsed observation is the constant $\perp$, although the occurrence of unparseability itself may still be informative about $\Theta$. By contrast, on parseable rounds the same raw
broadcast may transmit strictly positive information. In this precise sense,
the informational status of a signal is relative to the learner's structural
capacity to decode it.

This relativity is consistent with classical information theory. Randomness has
always been observer dependent: a ciphertext appears as pure noise without the
cryptographic key \citep{shannon1949secrecy}, and conditional mutual
information formalizes the dependence of information on what is known
\citep{cover}. What is distinctive here is the mechanism that generates
this dependence: the learner's decoding power is governed by the combinatorial
closure operator $\Phi_\mind$, so prerequisite topology directly determines
when the channel behaves as identity and when it behaves as erasure.

\medskip

The notion of \emph{mind-relative randomness} introduced earlier is related to,
but distinct from, the combinatorial distinction between ordered and unordered
concepts from \Cref{def:ordered}. The latter is a structural property of the
targeted concept relative to the learner's acquired concept set, whereas the
former is an epistemic property of the observation process relative to the
latent target $\Theta$.

In the sharp parsing model, if the teacher targets a concept $
C_{t+1}=\tgt(Z_{t+1})$
that is unordered at the current acquired concept set, $
C_{t+1}\notin \Phi_\mind(\mc K_t)$, 
then the parser maps every such raw signal to the same null observation: $
Y_{t+1}=\bot$ {almost surely on that event.}
Thus all distinctions among those raw signals are erased at the learner end of
the channel.

However, the appearance of $\bot$ does not by itself imply mind-relative
randomness. Even though the symbol $\bot$ contains no internal distinctions,
the event $\{Y_{t+1}=\bot\}$ may still convey information about the target
$\Theta$. In particular, if the teacher's targeting rule depends on $\Theta$,
then the probability that the teacher selects a concept outside
$\Phi_\mind(\mc K_t)$ may vary with $\Theta$, and observing $\bot$ can update
the learner's posterior belief.

Conversely, an ordered round need not be informative. If $
C_{t+1}\in \Phi_\mind(\mc K_t)$,
then the signal is parseable and $Y_{t+1}=Z_{t+1}$. But even in this case the
parsed observation may still be mind-random if, conditional on the public
history $\Fcal_t$, the teacher's policy induces the same distribution of
$Y_{t+1}$ under every possible target. Equivalently, $
\Theta \dperp Y_{t+1}\mid \Fcal_t$.
Thus parseability and informativeness are logically distinct: an unordered
round may still be informative through the occurrence of erasure, while an
ordered round may be uninformative if the parsed signal distribution does not
depend on the target.

An immediate consequence of sharp parsing is that repeated rephrasings of the
same unordered concept do not help. If $c\notin \Phi_\mind(\mc K_t)$, then
every raw signal targeting $c$ collapses to the null observation $\bot$,
regardless of how many distinct encodings or phrasings are available (see \Cref{cor:rephrasing}). Thus, on
that event, repetition and rephrasing do not reduce epistemic uncertainty about
the target.

Combined with the prerequisite gating established in \Cref{thm:relativity},
these observations formalize a central thesis of the framework: whether a
broadcast conveys usable information is not an intrinsic property of the signal
itself, but of the interaction among the signal, the learner's current
acquired concept set, and the teacher's policy.

\begin{remark}The relativity of randomness established in \Cref{thm:relativity} suggests a
broader perspective in which randomness itself becomes observer dependent.
The parsing map $\rho_\mind$ determines, for each mind and acquired state,
which signals are informative and which collapse to noise. In this sense,
randomness is not an intrinsic property of a signal but a relation between the
signal and the observer's structure of understanding.
\end{remark}

\section{Speed Limits of Teaching}
\label{sec:information-transfer-speed-teaching}
We now derive the quantitative speed limits of the teaching model. Two
obstructions coexist. The first is \emph{structural}: the learner must acquire
enough prerequisite concepts for the target to become reachable. The second is
\emph{epistemic}: the learner must resolve uncertainty about which target
concept the teacher intends. The purpose of this section is to formalize both
obstructions and combine them into a single lower bound on the expected
completion time.

Fix a mind $
\mind=(\mc C,\Acal_\mind,\mathcal E_\mind)$
and a finite target set $
\Omega\subseteq \mc U_\mind=\cl_\mind(\Acal_\mind)$.
Thus every target under consideration lies in the learner understanding
horizon.

\subsection{Identification and state-dependent capacity}
\label{sec:state-dep-cap}

Recall that $\Theta:\Omega_0\to\Omega$ is the realized target concept and that
the learner observes the parsed history $
\Fcal_t=\sigma(Y_1,\dots,Y_t)$.
The learner epistemic objective is to identify $\Theta$ from this history. We
say that identification occurs at time $t$ if $
\pi_t(\Theta)=1$,
equivalently, $
\mathds H(\Theta\mid \Fcal_t)=0$.
Since completion additionally requires target acquisition, identification is a
strictly weaker requirement than full teaching completion.

\begin{definition}[Identification stopping time]\label{def:id-time}
A random time $\tau_{\mathrm{id}}$ is an \emph{identification stopping time} if:
\begin{enumerate}[nosep,label=\textnormal{(\roman*)}]
\item $\tau_{\mathrm{id}}$ is an $(\Fcal_t)$-stopping time;
\item $\mathbb P(\tau_{\mathrm{id}}<\infty)=1$;
\item $\Theta$ is $\Fcal_{\tau_{\mathrm{id}}}$-measurable.
\end{enumerate}
Equivalently, $
\mathds H(\Theta\mid \Fcal_{\tau_{\mathrm{id}}})=0$ \text{almost surely.}
\end{definition}

Because the parsing map $\rho_\mind$ depends on the learner's current acquired
concept set $\mc K_t$, the effective learner-side channel is state dependent.
At early
stages many raw signals may collapse to $\bot$, whereas later the same signals
may pass through unchanged once the relevant prerequisites have been acquired.

Let $\mathbb K_\mind$ be the reachable family introduced in
\Cref{def:reachable}. For each reachable acquired concept set
$\mc K\in\mathbb K_\mind$, define the ordered raw-signal set $
\mathcal Z_{\mathrm{ord}}(\mc K)
=
\{z\in\mc Z:\tgt(z)\in \Phi_\mind(\mc K)\}$.
Under sharp parsing, signals in $\mathcal Z_{\mathrm{ord}}(\mc K)$ pass through
unchanged, while all other raw signals collapse to $\bot$.

This leads to the following learner-side capacity notion.

\begin{definition}[State-dependent parsed entropy bound]
\label{def:state-capacity}
For each $\mc K\in\mathbb K_\mind$, define
\[
C_{\mind}(\mc K)
=
\sup\left\{
\mathds H\bigl(\rho_\mind(Z,\mc K)\bigr)
:
Z \text{ is an }\mc Z\text{-valued random variable}
\right\}.
\]
\end{definition}
The parsed entropy bound $C_\mind(\mc K)$ also depends on the signal system
$(\mc Z,\tgt)$. Throughout the paper this instructional interface is treated
as fixed, and we therefore suppress this dependence in the notation. For a
given interface, the variation of $C_\mind(\mc K)$ across acquired concept
sets is endogenous to the learner state, whereas its numerical level is
determined jointly by the mind $\mind$ and the signal system $(\mc Z,\tgt)$.
Thus $C_\mind(\mc K)$ should be understood as a property of the pair
$(\mind,(\mc Z,\tgt))$ evaluated at state $\mc K$. 

Thus $C_\mind(\mc K)$ is the largest Shannon entropy that a one-round parsed
observation $\rho_\mind(Z,\mc K)$ can attain at the learner end of the channel
when the acquired concept set is $\mc K$, as the law of the raw input signal
$Z$ ranges over all $\mc Z$-valued distributions.

\begin{proposition}[Statewise one-round information bound]
\label{prop:state-info-bound}
For every $t\ge 0$, $
\mathds I(\Theta;Y_{t+1}\mid \Fcal_t)\le C_\mind(\mc K_t)$ \text{almost surely.}
\end{proposition}

\Cref{prop:state-info-bound} shows that the learner's per-round information
gain about the target is bounded by the capacity $C_\mind(\mc K_t)$, which
depends on the learner's acquired concept set at time $t$. As the learner
acquires more concepts, the set of parseable signals grows, and the capacity
may increase. The bound is therefore not static: structural progress expands
the effective channel through which teaching occurs. This coupling between
structural progress and informational capacity is the mechanism through which
prerequisites govern the speed of teaching.

The next lemma shows that the learner-side channel can only improve as the
learner acquires more concepts.

\begin{lemma}[Monotonicity of the state-dependent bound]
\label{lem:cap-mono}
If $\mc K,\mc K'\in\mathbb K_\mind$ satisfy $\mc K\subseteq \mc K'$, then $
C_\mind(\mc K)\le C_\mind(\mc K')$.
\end{lemma}

\Cref{lem:cap-mono} reflects a basic property of the parsing model: acquiring
additional concepts cannot reduce the learner's ability to decode signals.
When the acquired concept set grows, previously parseable signals remain
parseable, and additional signals may become usable. Consequently the entropy
of the parsed observation, and therefore the effective channel capacity,
cannot decrease as the learner acquires additional concepts.

This monotonicity admits a stronger statistical interpretation. To formalize
it, we use the Blackwell order on statistical experiments
\citep{ref:blackwell1}. Informally, one experiment Blackwell-dominates another if
the latter can be obtained from the former by garbling, that is, by
post-processing through a stochastic map independent of the underlying state.
Equivalently, the dominating experiment is at least as informative for every
statistical decision problem.

\begin{definition}[Blackwell domination]
Let $\Omega$ be a finite state space, and let
$\mathds W:\Omega\to\Delta(\mc Y)$ and
$\mathds W':\Omega\to\Delta(\mc Y')$
be two statistical experiments. We say that $\mathds W$
\emph{Blackwell-dominates} $\mathds W'$ if there exists a Markov kernel $
\mathbb G:\mc Y\to\Delta(\mc Y')$
such that for every $\omega\in\Omega$, $
\mathds W'(\cdot\mid \omega)
=
\sum_{y\in\mc Y} \mathbb G(\cdot\mid y)\mathds W(y\mid \omega)$.
Equivalently, $\mathds W'$ is a garbling of $\mathds W$.
\end{definition}
\begin{theorem}[Blackwell order on acquired concept sets]
\label{thm:blackwell}
Fix $t\ge 0$ and a public history $
h_t=(y_1,\dots,y_t)\in \mc Y^t$
with $
\mathbb P((Y_1,\dots,Y_t)=h_t)>0$.
For each $\mc K\in\mathbb K_\mind$, let $ \mathds W_{\mc K,h_t}$ denote the
statistical experiment from $\Theta$ to the parsed observation induced by the
conditional raw-signal law
\[
\mathbb P\bigl(Z_{t+1}\in\cdot\mid \Theta=\omega,\,(Y_1,\dots,Y_t)=h_t\bigr),
\qquad \omega\in\Omega.
\]
If $\mc K\subseteq \mc K'$, then $\mathds W_{\mc K',h_t}$
Blackwell-dominates $\mathds W_{\mc K,h_t}$.
\end{theorem}

\Cref{thm:blackwell} holds for each realized public history $h_t$
separately. Thus the ordering of acquired concept sets is pathwise rather than
merely averaged: conditional on any history for which the next-round raw-signal
law is defined, the parsed experiment induced by a larger acquired concept set
Blackwell-dominates the parsed experiment induced by a smaller one.

This theorem strengthens \Cref{lem:cap-mono}. The monotonicity of
$C_\mind(\mc K)$ says that larger acquired concept sets permit weakly greater
parsed entropy. \Cref{thm:blackwell} shows more: they induce uniformly more
informative experiments in the sense of statistical decision theory. The
universal-broadcast theorem of \Cref{thm:impossibility} will show that this
dependence on the learner prerequisite structure cannot, in general, be
eliminated by a common broadcast curriculum.
\subsection{Structural and epistemic lower bounds}
\label{sec:structural-epistemic}
We now combine the structural and epistemic constraints of the model to derive a single lower bound on teaching time.

\begin{definition}[Structural distance to a target concept]
\label{def:structural-distance}
For $c\in \mc U_\mind$, define
\[
L_\mind(c)
=
\min\Bigl\{
L\ge 0:
\exists\,\mc K_0,\dots,\mc K_L\in \mathbb K_\mind,\;
u_0,\dots,u_{L-1}\in \mc U_\mind
\text{ such that }
\]
\[
\mc K_0=\Acal_\mind,~
c\in \mc K_L,~
\mc K_{i+1}=\mc K_i\cup\{u_i\},~
u_i\in \Phi_\mind(\mc K_i)\setminus \mc K_i, 
\quad i=0,\dots,L-1
\Bigr\}.
\]
\end{definition}

The quantity $L_\mind(c)$ measures the shortest prerequisite-respecting route
from the axioms to a state containing $c$. It therefore gives the natural
structural benchmark against which any completion time must be compared. The first fundamental constraint on teaching time is structural: the learner
must traverse the prerequisite chain before the target can be acquired.  
\begin{proposition}[Structural barrier]\label{prop:structural-barrier}
Let $\tau$ be any completion time in the sense of \Cref{def:completion}. Then $
\tau \ge L_\mind(\Theta)$ almost surely.
Consequently, $
\mathbb E[\tau]\ge \mathbb E[L_\mind(\Theta)]$.
\end{proposition}
\Cref{prop:structural-barrier} is the purely geometric obstruction in the
model. Regardless of how informative the signals are, the learner cannot
complete teaching before traversing a prerequisite-respecting path to a state
containing the realized target. Since each round adds at most one concept, the
shortest such path gives an unavoidable lower bound on completion time.

To control the epistemic obstruction, we next aggregate the information gained
across all rounds up to identification $\tau_{\mathrm{id}}$. The point is that the per-round entropy
drop identity from \Cref{prop:first-law} telescopes over time.

\begin{lemma}[Total information required for identification]
\label{lem:chain}
Let $\tau_{\mathrm{id}}$ be an identification stopping time. Then
\[
\mathbb E\left[\sum_{t=0}^{\tau_{\mathrm{id}}-1}
\mathds I(\Theta;Y_{t+1}\mid \Fcal_t)\right]
=
\mathds H(\Theta).
\]
\end{lemma}
\Cref{lem:chain} says that identification must pay for the full initial
uncertainty of the target. The cumulative conditional mutual information
transmitted through the parsed observations up to identification is the
entropy of $\Theta$. Thus the learner cannot identify the target until enough
usable information has flowed through the learner-side channel to resolve all
initial uncertainty.

The next step is to combine this accounting identity with the statewise capacity
bound from \Cref{prop:state-info-bound}. This converts total required
information into a lower bound expressed in terms of the learner trajectory
through the reachable family.

\begin{proposition}[Trajectory information budget]
\label{prop:trajectory-budget}
Let $\tau_{\mathrm{id}}$ be any identification stopping time. Then
\[
\mathds H(\Theta)
\le
\mathbb E\left[\sum_{t=0}^{\tau_{\mathrm{id}}-1} C_\mind(\mc K_t)\right].
\]
\end{proposition}

\Cref{prop:trajectory-budget} is the dynamic information budget of the model.
The total target uncertainty cannot exceed the cumulative parsed capacity along
the states visited before identification. In this sense, a curriculum may need
to spend rounds building the decoder before it can effectively use it: the
states through which the learner passes determine the rate at which target
information can be transmitted.

We define $
C_\mind^{\max}
=
\max_{\mc K\in\mathbb K_\mind} C_\mind(\mc K)$. 
This maximum is well defined because, under \Cref{ass:finite}, the reachable
family $\mathbb K_\mind$ is finite by \Cref{prop:reach-structure}.
\begin{theorem}[Global structural-information lower bound]
\label{thm:struct-info-bound}
Let $\tau$ be any completion time, 
then
\[
\mathbb E[\tau]
\ge
\max\left\{
\mathbb E[L_\mind(\Theta)],
\frac{\mathds H(\Theta)}{C_\mind^{\max}}
\right\}.
\]
\end{theorem}

\Cref{thm:struct-info-bound} is the central speed law of the framework.
Teaching is constrained simultaneously by prerequisite geometry and by
information transmission. The lower bound takes the form of a maximum rather
than an additive sum because structural progress may itself convey information
about the target. Nevertheless both bottlenecks must be cleared.

\begin{assumption}[Structural signal availability]
\label{ass:structural-signals}
For every concept $u\in \mc U_\mind$, there exists a raw signal $z\in\mc Z$
such that $
\tgt(z)=u$.
\end{assumption}
Earlier we required only that every possible target concept admit a
corresponding signal, that is, $\Omega\subseteq \operatorname{im}(\tgt)$.
\Cref{ass:structural-signals} is stronger: it requires signals for all concepts
in the understanding horizon $\mc U_\mind$, including intermediate
prerequisites. This assumption ensures that the teacher can implement any
valid ordered curriculum by emitting signals targeting the concepts that must
be acquired along the path to the target.

\begin{proposition}[Direct target signaling collapses the epistemic term in the baseline model]
\label{prop:direct-target-collapse}
Let $
\Omega_+=\{c\in\Omega:\pi_0(c)>0\}$
be the support of the prior.

\begin{enumerate}[nosep,label=\textnormal{(\roman*)}]
\item If $\Omega\subseteq \operatorname{im}(\tgt)$, then $
\frac{\mathds H(\Theta)}{C_\mind^{\max}}\le 1$.
\item Under \Cref{ass:structural-signals}, there exists an admissible teaching
strategy with completion time $\tau$ satisfying $
\tau\le L_\mind(\Theta)+1$ \text{almost surely,}
and hence $
\mathbb E[\tau]\le \mathbb E[L_\mind(\Theta)] + 1$.
\end{enumerate}
\end{proposition}

\Cref{prop:direct-target-collapse} clarifies the role of the
information-theoretic layer in the baseline model. Once the learner has
structurally reached the realized target, a single target-specific signal
suffices for identification. Thus the dominant obstruction is typically
structural: the learner must first acquire the prerequisites that make the
target concept reachable.

The information-theoretic analysis nevertheless remains essential. It explains
why target-specific instruction is ineffective before the relevant prerequisites
are in place, and it provides a principled way to compare the informativeness of
different acquired states through Blackwell dominance. In this view, structural progress builds the decoder, and information
transmission becomes effective only after that decoder exists.

\begin{example}[A common prerequisite can open a parseable identification channel]
\label{ex:decoder-building}
Consider the mind $
\mind=(\mc C,\Acal_\mind,\mathcal E_\mind)$
with
\[
\mc C=\{a,b,d_1,d_2,d_3,d_4\},
\qquad
\Acal_\mind=\{a\},
\]
and expansion rules $
\{a\}\Rightarrow b$,
$\{b\}\Rightarrow d_j$, $j=1,2,3,4$.
Thus $b$ is a common prerequisite, and once $b$ has been acquired, any of the
four target concepts $d_1,d_2,d_3,d_4$ becomes reachable in one additional
step.
Let $
\Omega=\{d_1,d_2,d_3,d_4\}$
with the uniform prior. Then $
\mathds H(\Theta)=\log 4=2$.
For each $j=1,2,3,4$, $
L_\mind(d_j)=2$,
and therefore $
\mathbb E[L_\mind(\Theta)]=2$.
Let the raw signal alphabet be $
\mc Z=\{z_b,z_1,z_2,z_3,z_4\}$,
with $
\tgt(z_b)=b$,
$\tgt(z_j)=d_j,~j=1,2,3,4$.

At the initial acquired concept set $\{a\}$, only $b$ is ordered. Hence the
parsed observation range is $
\{z_b,\bot\}$,
so $
C_\mind(\{a\})=\log 2$.
At the acquired concept set $\{a,b\}$, all five raw signals are parseable, so
the parsed observation range is $
\{z_b,z_1,z_2,z_3,z_4\}$,
and therefore $
C_\mind(\{a,b\})=\log 5$.
Thus acquiring the single prerequisite $b$ enlarges the learner effective
channel from $1$ bit to $\log 5$ bits per round.

Now suppose the teacher tries to identify the target immediately by sending
\[
Z_1=z_j
\qquad\text{when }\Theta=d_j.
\]
At the raw-signal level this would reveal the target perfectly. But at the
learner initial acquired concept set $\{a\}$, none of the targets $d_j$ is
ordered, so
\[
Y_1=\rho_\mind(Z_1,\{a\})=\bot
\qquad\text{almost surely.}
\]
Hence $
\mathds I(\Theta;Y_1\mid \Fcal_0)=0$.
Before the common prerequisite $b$ is taught, target-specific instruction is
pure erasure.

Consider instead the two-round strategy
\[
Z_1=z_b
\qquad\text{for every realization of }\Theta,
\]
and
\[
Z_2=z_j
\qquad\text{if }\Theta=d_j.
\]
After round $1$, the learner has acquired the prerequisite: $
\mc K_1=\{a,b\}$.
The posterior does not change, because the first signal is independent of
$\Theta$. At round $2$, the signal $z_j$ is parseable, so the learner observes
$Y_2=z_j$, acquires $d_j$, and identifies the target exactly. Thus $
\tau=2$ {almost surely.}

The lower bound of \Cref{thm:struct-info-bound} is therefore tight in this
example. Since $
C_\mind^{\max}=\log 5$,
one obtains
\[
\mathbb E[\tau]
\ge
\max\left\{
\mathbb E[L_\mind(\Theta)],
\frac{\mathds H(\Theta)}{C_\mind^{\max}}
\right\}
=
\max\left\{
2,\frac{2}{\log 5}
\right\}
=
2,
\]
and the strategy above attains equality.

This example shows that an optimal teacher may rationally spend an entire round
on structural preparation rather than on target-specific signaling, because
target-specific signals are useless before the common prerequisite $b$ has been
acquired. In the baseline model, the information-theoretic term is not the
binding lower bound here, since \Cref{prop:direct-target-collapse} implies that
identification costs at most one additional round once the target is
structurally reachable. The example nevertheless illustrates the central
mechanism of the section: usable information is state dependent, and teaching
may need to enlarge the learner parsed alphabet before target information can
flow.
\end{example}

\section{Structural Limits on Teaching}
\label{sec:structural-limits}

This section develops two consequences of the structural view of teaching.
First, for a fixed learner mind, the prerequisite geometry creates
threshold effects in finite-horizon teaching: below a critical time budget,
completion is impossible for every strategy, while beyond that threshold
success becomes feasible and, under mild assumptions, eventually likely.
Second, for heterogeneous learners, structural incompatibilities generate an
intrinsic inefficiency of universal broadcast curricula: a single common
sequence of signals may be forced to pay separately for prerequisites that
personalized teaching would handle individually.

Taken together, these results show that the limits of teaching are not merely
informational. They are already encoded in the combinatorial structure of the
learner prerequisite system. That structure determines both when teaching can
begin to succeed and how costly common instruction becomes across different
minds.

\subsection{Structural thresholds in teaching}
\label{sec:value}
A central question in any teaching problem is: \textit{given a fixed time
budget, what is the probability that teaching succeeds?} The prerequisite
structure of the learner determines the answer.
Below a certain
threshold, completion is impossible for every teaching strategy. Once the time
horizon exceeds that threshold, completion is no longer ruled out a priori,
and under mild assumptions the optimal fixed-horizon success probability
converges to one as the horizon grows.

This vanishing completion probability is not an approximation but a direct consequence of the structural barrier. It also has an immediate implication for resource allocation: when training budgets are scarce, distributing time evenly across learners may produce no completed learners at all, whereas concentrating the same budget on fewer learners can yield strictly positive output.

Recall the stochastic teaching model from \Cref{sec:stoc-teaching}. The target
concept $\Theta$ is drawn from a prior $\pi_0$ on $\Omega$, known to both
teacher and learner. By \Cref{def:completion}, teaching is complete at the
random time $\tau$ if both
\begin{enumerate}[nosep, label=\textnormal{(\roman*)}]
\item the learner has acquired the target concept, $\Theta\in \mc K_\tau$;
\item the learner has identified the target, $\pi_\tau(\Theta)=1$.
\end{enumerate}

We therefore ask: if the teacher is given a budget of $t$ rounds, what is the
maximal probability of completing teaching within that budget? Define
\[
\mathds V(t)
=
\sup_{\text{admissible teaching strategies}}
\mathbb P(\tau\le t).
\]
Thus $\mathds V(t)$ is the optimal success probability achievable with a time
budget of $t$ rounds, computed under the prior on $\Theta$.

Recall also that for each target $c\in\Omega$, the quantity $L_\mind(c)$
denotes the structural distance from the axiom set $\Acal_\mind$ to a reachable
acquired concept set containing $c$. Define
\[
L_{\min}
=
\min\{L_\mind(c):\pi_0(c)>0\}.
\]
This is the smallest structural distance among targets that can arise under the
prior.

For expected completion time, the baseline model also admits the upper bound
$\mathbb E[\tau]\le \mathbb E[L_\mind(\Theta)] + 1$ under
\Cref{ass:structural-signals} (\Cref{prop:direct-target-collapse}). The
fixed-horizon analysis below complements that statement by describing the
threshold structure of success probabilities as a function of the time budget.

\begin{proposition}[Zero completion below the structural threshold]
\label{prop:hard-zero}
For every $t\in\mathbb N$,
\[
\mathds V(t)\le \mathbb P\bigl(L_\mind(\Theta)\le t\bigr).
\]
In particular, $
\mathds V(t)=0$ for all $t<L_{\min}$.
\end{proposition}

\Cref{prop:hard-zero} shows that if the time budget is shorter than the
structural depth of every possible target, then completion is impossible. No
teaching strategy can circumvent this obstruction, because the learner cannot
be moved to a reachable acquired concept set containing the realized target in
so few rounds.

At the opposite extreme, if some admissible strategy completes teaching in
finite expected time, then the optimal fixed-horizon success probability
converges to one as the horizon grows.

\begin{proposition}[Eventual success]
\label{prop:eventual-success}
If there exists an admissible teaching strategy such that $
\mathbb E[\tau]<\infty$,
then $
\mathds V(t)\to 1$ as $t\to\infty$.
More concretely, for any such strategy,
\[
\mathds V(t)\ge 1-\frac{\mathbb E[\tau]}{t}
\qquad\text{for all } t\ge 1.
\]
\end{proposition}

Together, \Cref{prop:hard-zero,prop:eventual-success} describe the qualitative
shape of the fixed-horizon success function $\mathds V(t)$: an initial region
of structural impossibility, followed by a region in which success becomes
increasingly likely as the time budget grows.

To make the allocation implications transparent, it is useful to consider the
special case of a deterministic target.

Let $g\in \mc U_\mind$ be fixed, and suppose that $
\Theta=g$ \text{almost surely.}
Then the prior is degenerate, so $
\pi_t=\delta_g$ for all $t$.
Hence identification is automatic, and completion reduces to target acquisition
alone.

Define the target-acquisition time of $g$ by $
\tau_g=\inf\{t\ge 0:g\in \mc K_t\}$,
and define the fixed-horizon acquisition probability by
\[
\mathds V_g(t)
=
\sup_{\text{admissible teaching strategies}}
\mathbb P(\tau_g\le t).
\]

\begin{proposition}[Step function for deterministic targets]
\label{prop:deterministic-step}
Assume that the parsing map is given by \Cref{def:parsing} and that
\Cref{ass:structural-signals} holds. Then for every deterministic target
$g\in \mc U_\mind$,
\[
\mathds V_g(t)=
\begin{cases}
0, & {\rm if }~ t<L_\mind(g),\\[0.3em]
1, & {\rm otherwise. }
\end{cases}
\]
\end{proposition}

Thus, for a deterministic target, the fixed-horizon acquisition probability is
a step function at the structural distance $L_\mind(g)$. Below that
threshold acquisition is impossible; at and above it, acquisition can be
achieved with certainty.

\begin{remark}
The threshold structure above contrasts with benchmark models of human-capital
accumulation in which training is represented by a smooth production technology
for human capital (\textit{e.g.}, \citep{ref:benporath1967,ref:becker1964}). In
such models every marginal unit of investment yields a positive, though
possibly diminishing, return. In the present framework, prerequisite-gated
learning induces a threshold technology: a teaching signal has no effect until
the learner prerequisite structure admits the target concept, after which
additional signals become productive. The induced production technology is
therefore \textit{non}-concave.
\end{remark}

This threshold structure has direct implications for the allocation of training
resources. Consider a decision maker who must allocate a fixed instructional
budget across learners, for example a firm training workers in a specific skill
or an instructor allocating tutoring hours across students. The planner has a
total budget of $B$ instructional rounds and must decide how to distribute them
across $N$ learners.

\begin{proposition}[Allocation under structural thresholds]
\label{prop:capital-destruction}
Assume that the parsing map is given by \Cref{def:parsing} and that
\Cref{ass:structural-signals} holds. Fix a deterministic target $
g\in \mc U_\mind$
with $
L=L_\mind(g)\ge 1$.
Consider $N$ identical learners and a total budget of $B\in\mathbb N$
instructional rounds.
\begin{enumerate}[nosep, label=\textnormal{(\roman*)}]
\item Any allocation that gives every learner fewer than $L$ rounds yields zero
completed learners.
\item There exists an allocation that gives $L$ rounds to
$\min\{N, \lfloor B/L\rfloor\}$ learners and $0$ rounds to the remaining learners, and
under this allocation $
\min \{N , \left\lfloor B/L\right\rfloor\}$
learners complete.
\end{enumerate}

In particular, if $B<NL$ and the budget is spread so that every learner
receives fewer than $L$ rounds, then total output is zero, whereas the
concentrated allocation in \textnormal{(ii)} yields strictly positive output
whenever $B\ge L$.
\end{proposition}

\Cref{prop:capital-destruction} shows that evenly spreading a fixed training
budget can waste the entire budget when every learner remains below the
structural threshold. By contrast, concentrating the same budget on fewer
learners allows those learners to cross the threshold and produce strictly
positive output. The source of this effect is structural: for a deterministic
target, additional training time has no effect until the prerequisite
threshold $L_\mind(g)$ is reached, at which point completion becomes possible.
The zero-output region is therefore not imposed from outside the model but is a
direct consequence of the learner prerequisite geometry.

For random targets, the step-function structure need not persist,
because different targets may have different structural depths. What remains is
the zero-completion phenomenon from \Cref{prop:hard-zero}: if every learner
receives fewer than $
L_{\min}=\min\{L_\mind(c):\pi_0(c)>0\}$,
rounds, then the completion probability is zero regardless of the teaching
strategy. The qualitative allocation lesson therefore extends beyond the
deterministic case: if the available budget is spread so thinly that every
learner remains below the relevant structural threshold, no learner completes.

\subsection{Limits of universal broadcast curricula}
\label{sec:impossibility}

The preceding subsection concerned a single learner mind. We now turn to
\textit{heterogeneous} learners whose prerequisite structures differ. In that setting, a
teacher restricted to a single broadcast curriculum cannot adapt instruction to
individual minds. The next theorem shows that this restriction carries a
structural cost: even when each learner can be taught efficiently by a
personalized curriculum, any common broadcast may be forced to pay a linear
penalty in the number of learner types.

\begin{theorem}[Linear broadcast penalty for incompatible minds]
\label{thm:impossibility}
Fix integers $k\ge 2$ and $L\ge 2$. Then one can construct
\begin{itemize}[nosep]
\item a finite concept space $\mc C$,
\item a common axiom set $\Acal\subseteq \mc C$,
\item a finite raw-signal alphabet $\mc Z$ together with a signal target map $
\tgt:\mc Z\to \mc C$,
\item minds $\mind_1,\dots,\mind_k$ on $\mc C$ with common axiom set $
\Acal_{\mind_i}=\Acal$, $i=1,\dots,k$,
but pairwise distinct rule sets $\mathcal E_{\mind_i}$,
\item and a common deterministic target concept $g\in \mc C$,
\end{itemize}
such that:
\begin{enumerate}[label=\textnormal{(\roman*)}]
\item for each $i\in\{1,\dots,k\}$, there exists a valid ordered curriculum for
$\mind_i$ of length $L$ whose final acquired concept set contains $g$;
\item if a common broadcast sequence $
\Gamma=(z_1,\dots,z_T)\in \mc Z^T$
is presented to all $k$ minds, and if the induced acquired concept processes
start from $
\mc K_0^{(i)}=\Acal$, $i=1,\dots,k$,
and evolve according to
\[
\mc K_{t+1}^{(i)}
=
\begin{cases}
\mc K_t^{(i)}\cup\{\tgt(z_{t+1})\},
& \text{if }\tgt(z_{t+1})\in \Phi_{\mind_i}(\mc K_t^{(i)}),\\[1mm]
\mc K_t^{(i)},
& \text{otherwise,}
\end{cases}
\qquad t=0,\dots,T-1,
\]
then the condition $
g\in \mc K_T^{(i)}$ for every $ i=1,\dots,k$
implies $
T\ge k(L-1)+1;$
\item There exists a common broadcast sequence of
length $k(L-1)+1$ for which $
g\in \mc K_{k(L-1)+1}^{(i)}$ for every $ i=1,\dots,k$.
\end{enumerate}
\end{theorem}

\Cref{thm:impossibility} is an existence result. For any prescribed number
$k$ of learner types and any prescribed personalized teaching length $L$, one
can construct $k$ minds sharing the same axiom set and the same deterministic
target concept $g$, but having different prerequisite structures. For each
mind, the target can be acquired in $L$ personalized rounds. However,
every common broadcast sequence that succeeds for all minds must have length at
least $
k(L-1)+1$.

The source of the penalty is purely structural. Each mind possesses a
private prerequisite chain leading to the target concept, and signals that advance one
mind along its chain are unparseable for the others. Consequently a
universal broadcast cannot reuse prerequisite rounds across learner
types; it must effectively pay for each private chain separately. This
is what generates the linear dependence of the required broadcast length
on the number of learner types.

\paragraph{Acknowledgments.}
The author thanks Teng Andrea Xu for helpful discussions.

\begin{center}
  \vspace*{2\baselineskip} 
  {\Huge \bfseries Appendix} 
  \vspace{2\baselineskip}   
\end{center}

\addtocontents{toc}{\protect\setcounter{tocdepth}{3}}

\noindent 
This document contains supplementary material for the paper
\emph{A Mathematical Theory of Understanding}.

{
  \renewcommand{\contentsname}{} 
  \vspace*{-2em}                 
  \tableofcontents
}
\appendix

\section{Proofs}
\label{app:proofs}

This appendix collects the proofs omitted from the main text, organized by the
section in which the corresponding result appears. Supplementary lemmas are used in the proofs but not stated in the main text are
included where they arise.

\subsection{Proofs for Section~\ref{sec:axioms}}

\subsubsection*{Proof of \Cref{lem:mono}}
Let $c\in \Phi_\mind(\mc K)$. If $c\in \mc K$, then $c\in \mc K'\subseteq \Phi_\mind(\mc K')$. Otherwise there exists $(\mc S,c)\in \mathcal E_\mind$ with $\mc S\subseteq \mc K$. Since $\mc K\subseteq \mc K'$, one also has $\mc S\subseteq \mc K'$, hence $c\in \Phi_\mind(\mc K')$.
\qed

\begin{lemma}[Directed-union continuity]\label{lem:finitary}
If $(\mc K_\alpha)_{\alpha\in\mc A}$ is a nonempty directed family, then $
\Phi_\mind\left(\bigcup_{\alpha\in\mc A}\mc K_\alpha\right)
=
\bigcup_{\alpha\in\mc A}\Phi_\mind(\mc K_\alpha)$.
\end{lemma}
\begin{proof}
The inclusion $\supseteq$ follows from monotonicity of $\Phi_\mind$.

For the reverse inclusion, let $c\in \Phi_\mind(\bigcup_\alpha \mc K_\alpha)$. If $c\in \bigcup_\alpha \mc K_\alpha$, then $c\in \Phi_\mind(\mc K_\alpha)$ for some $\alpha \in \mc A$.

Otherwise there exists a finite set $\mc S\subseteq \bigcup_\alpha \mc K_\alpha$ with $(\mc S,c)\in \mathcal E_\mind$. For each $s\in \mc S$, choose $\alpha_s$ such that $s\in \mc K_{\alpha_s}$. Because the family is directed and $\mc S$ is finite, there exists $\gamma$ with $
\bigcup_{s\in\mc S}\mc K_{\alpha_s}\subseteq \mc K_\gamma$.
Hence $\mc S\subseteq \mc K_\gamma$, so $c\in \Phi_\mind(\mc K_\gamma)$.
\end{proof}

\subsubsection*{Proof of \Cref{prop:closure-char}}
(i) We first show that the collection of fixed points of $\Phi_{\mathfrak{m}}$ containing $\mathcal{K}$ is non-empty. Consider the entire concept space $\mathcal{C}$. By extensiveness, $\mathcal{C} \subseteq \Phi_{\mathfrak{m}}(\mathcal{C})$. For the reverse inclusion: every expansion rule $(\mathcal{S}, c) \in \mathcal{E}_{\mathfrak{m}}$ has $c \in \mathcal{C}$ by definition, so $\Phi_{\mathfrak{m}}(\mathcal{C}) \subseteq \mathcal{C}$. Together, $\Phi_{\mathfrak{m}}(\mathcal{C})=\mathcal{C}$, and clearly $\mc{K} \subseteq \mathcal{C}$. So $\mathcal{C}$ is a fixed point containing $\mathcal{K}$.

By \Cref{def:closure}, $\operatorname{cl}_{\mathfrak{m}}(\mathcal{K})=\bigcap\left\{\mathcal{F} \subseteq \mathcal{C}: \mathcal{K} \subseteq \mathcal{F}, \Phi_{\mathfrak{m}}(\mathcal{F})=\mathcal{F}\right\}$. The intersection is over a non-empty collection, so $\mathrm{cl}_{\mathfrak{m}}(\mathcal{K})$ is well-defined. We now show it is itself a fixed point. For any fixed point $\mathcal{F} \supseteq \mathcal{K}$ in the collection, $\operatorname{cl}_{\mathfrak{m}}(\mathcal{K}) \subseteq \mathcal{F}$, so monotonicity gives $\Phi_{\mathfrak{m}}\left(\operatorname{cl}_{\mathfrak{m}}(\mathcal{K})\right) \subseteq \Phi_{\mathfrak{m}}(\mathcal{F})=\mathcal{F}$. Since this holds for every such $\mathcal{F}$, we get $\Phi_{\mathfrak{m}}\left(\operatorname{cl}_{\mathfrak{m}}(\mathcal{K})\right) \subseteq \mathrm{cl}_{\mathfrak{m}}(\mathcal{K})$. Extensiveness gives the other direction: $\operatorname{cl}_{\mathfrak{m}}(\mathcal{K}) \subseteq \Phi_{\mathfrak{m}}\left(\operatorname{cl}_{\mathfrak{m}}(\mathcal{K})\right)$. Together: $\Phi_{\mathfrak{m}}\left(\operatorname{cl}_{\mathfrak{m}}(\mathcal{K})\right)=\operatorname{cl}_{\mathfrak{m}}(\mathcal{K})$.

\noindent (ii) By extensiveness, the sequence $\mathcal{K} \subseteq \Phi_{\mathfrak{m}}(\mathcal{K}) \subseteq \Phi_{\mathfrak{m}}^2(\mathcal{K}) \subseteq \cdots$ is non-decreasing. Let $\mathcal{L}=\cup_{n=0}^{\infty} \Phi_{\mathfrak{m}}^n(\mathcal{K})$. By \Cref{lem:finitary}:

\[
\Phi_{\mathfrak{m}}(\mathcal{L})=\Phi_{\mathfrak{m}}\left(\bigcup_{n=0}^{\infty} \Phi_{\mathfrak{m}}^n(\mathcal{K})\right)=\bigcup_{n=0}^{\infty} \Phi_{\mathfrak{m}}^{n+1}(\mathcal{K})=\mathcal{L}
\]

so $\mathcal{L}$ is a fixed point of $\Phi_{\mathfrak{m}}$ containing $\mathcal{K}$. Since $\operatorname{cl}_{\mathfrak{m}}(\mathcal{K})$ is the least such fixed point, $\operatorname{cl}_{\mathfrak{m}}(\mathcal{K}) \subseteq \mathcal{L}$. Conversely, we show by induction that $\Phi_{\mathfrak{m}}^n(\mathcal{K}) \subseteq \operatorname{cl}_{\mathfrak{m}}(\mathcal{K})$ for all $n \geq 0$. Base case: $\Phi_{\mathfrak{m}}^0(\mathcal{K})=\mathcal{K} \subseteq \mathrm{cl}_{\mathfrak{m}}(\mathcal{K})$ by definition. Inductive step: $\operatorname{suppose} \Phi_{\mathfrak{m}}^n(\mathcal{K}) \subseteq \mathrm{cl}_{\mathfrak{m}}(\mathcal{K})$. Since $\mathrm{cl}_{\mathfrak{m}}(\mathcal{K})$ is a fixed point, $\Phi_{\mathfrak{m}}\left(\mathrm{cl}_{\mathfrak{m}}(\mathcal{K})\right)=\mathrm{cl}_{\mathfrak{m}}(\mathcal{K})$. Monotonicity then gives $\Phi_{\mathfrak{m}}^{n+1}(\mathcal{K})=\Phi_{\mathfrak{m}}\left(\Phi_{\mathfrak{m}}^n(\mathcal{K})\right) \subseteq \Phi_{\mathfrak{m}}\left(\mathrm{cl}_{\mathfrak{m}}(\mathcal{K})\right)=\mathrm{cl}_{\mathfrak{m}}(\mathcal{K})$. Since $\Phi_{\mathfrak{m}}^n(\mathcal{K}) \subseteq \operatorname{cl}_{\mathfrak{m}}(\mathcal{K})$ for every $n$, we get $\mathcal{L}=\cup_{n=0}^{\infty} \Phi_{\mathfrak{m}}^n(\mathcal{K}) \subseteq \operatorname{cl}_{\mathfrak{m}}(\mathcal{K})$.

(iii) If $\mathcal{C}$ is finite, the chain $\mathcal{K} \subseteq \Phi_{\mathfrak{m}}^1(\mathcal{K}) \subseteq \Phi_{\mathfrak{m}}^2(\mathcal{K}) \subseteq \cdots$ is an increasing sequence of subsets of $\mathcal{C}$. Whenever $\Phi_{\mathfrak{m}}^{n+1}(\mathcal{K}) \neq \Phi_{\mathfrak{m}}^n(\mathcal{K})$, the inclusion is strict, so $\left|\Phi_{\mathfrak{m}}^{n+1}(\mathcal{K})\right| \geq\left|\Phi_{\mathfrak{m}}^n(\mathcal{K})\right|+1$. Since each set has at most $|\mathcal{C}|$ elements, strict growth can occur at most $|\mathcal{C}|-|\mathcal{K}| \leq|\mathcal{C}|$ times. Therefore $\Phi_{\mathfrak{m}}^N(\mathcal{K})=\Phi_{\mathfrak{m}}^{N+1}(\mathcal{K})$ for some $N \leq|\mathcal{C}|$, and the chain stabilizes: $\operatorname{cl}_{\mathfrak{m}}(\mathcal{K})=\Phi_{\mathfrak{m}}^N(\mathcal{K})$.
\qed

\subsubsection*{Proof of \Cref{prop:U-unique}}
We first show that $\mathcal{U}_{\mathfrak{m}}=\mathrm{cl}_{\mathfrak{m}}\left(\mathcal{A}_{\mathfrak{m}}\right)$ satisfies (i) to (iii).
(i) By \Cref{prop:closure-char}(ii), $\mathrm{cl}_{\mathfrak{m}}\left(\mathcal{A}_{\mathfrak{m}}\right)=\bigcup_{n \geq 0} \Phi_{\mathfrak{m}}^n\left(\mathcal{A}_{\mathfrak{m}}\right)$. Since $\Phi_{\mathfrak{m}}^0\left(\mathcal{A}_{\mathfrak{m}}\right)=\mathcal{A}_{\mathfrak{m}}$, it follows that $\mathcal{A}_{\mathfrak{m}} \subseteq \mathrm{cl}_{\mathfrak{m}}\left(\mathcal{A}_{\mathfrak{m}}\right)$.
(ii) Let $(\mathcal{S}, c) \in \mathcal{E}_{\mathfrak{m}}$ and suppose $\mathcal{S} \subseteq \mathrm{cl}_{\mathfrak{m}}\left(\mathcal{A}_{\mathfrak{m}}\right)$. By \Cref{prop:closure-char}(i), $\mathrm{cl}_{\mathfrak{m}}\left(\mathcal{A}_{\mathfrak{m}}\right)$ is a fixed point of $\Phi_{\mathfrak{m}}$, so $\Phi_{\mathfrak{m}}\left(\operatorname{cl}_{\mathfrak{m}}\left(\mathcal{A}_{\mathfrak{m}}\right)\right)=\operatorname{cl}_{\mathfrak{m}}\left(\mathcal{A}_{\mathfrak{m}}\right)$. Since $(\mathcal{S}, c) \in \mathcal{E}_{\mathfrak{m}}$ and $\mathcal{S} \subseteq \operatorname{cl}_{\mathfrak{m}}\left(\mathcal{A}_{\mathfrak{m}}\right)$, the definition of $\Phi_{\mathfrak{m}}$ gives $c \in \Phi_{\mathfrak{m}}\left(\mathrm{cl}_{\mathfrak{m}}\left(\mathcal{A}_{\mathfrak{m}}\right)\right)=\mathrm{cl}_{\mathfrak{m}}\left(\mathcal{A}_{\mathfrak{m}}\right)$.
(iii) Let $\mathcal{F} \subseteq \mathcal{C}$ satisfy (i) and (ii). Then $\mathcal{A}_{\mathfrak{m}} \subseteq \mathcal{F}$. Moreover, if $c \in \Phi_{\mathfrak{m}}(\mathcal{F})$, then either $c \in \mathcal{F}$, or there exists $(\mathcal{S}, c) \in \mathcal{E}_{\mathfrak{m}}$ with $\mathcal{S} \subseteq \mathcal{F}$, in which case (ii) gives $c \in \mathcal{F}$. Hence $\Phi_{\mathfrak{m}}(\mathcal{F}) \subseteq \mathcal{F}$. By extensiveness of $\Phi_{\mathfrak{m}}$, we also have $\mathcal{F} \subseteq \Phi_{\mathfrak{m}}(\mathcal{F})$. Therefore $\Phi_{\mathfrak{m}}(\mathcal{F})=\mathcal{F}$, so $\mathcal{F}$ is a fixed point of $\Phi_{\mathfrak{m}}$ containing $\mathcal{A}_{\mathfrak{m}}$. Since $\operatorname{cl}_{\mathfrak{m}}\left(\mathcal{A}_{\mathfrak{m}}\right)$ is the intersection of all such fixed points, we conclude that $\operatorname{cl}_{\mathfrak{m}}\left(\mathcal{A}_{\mathfrak{m}}\right) \subseteq \mathcal{F}$. Thus $\operatorname{cl}_{\mathfrak{m}}\left(\mathcal{A}_{\mathfrak{m}}\right)$ is the smallest set satisfying (i) and (ii).

Finally, suppose $\mathcal{U}$ and $\mathcal{U}^{\prime}$ both satisfy (i)-(iii). Since $\mathcal{U}^{\prime}$ satisfies (i) and (ii), the minimality property (iii) for $\mathcal{U}$ implies $\mathcal{U} \subseteq \mathcal{U}^{\prime}$. By symmetry, $\mathcal{U}^{\prime} \subseteq \mathcal{U}$. Hence $\mathcal{U}=\mathcal{U}^{\prime}$.
\qed

\subsubsection*{Proof of \Cref{thm:closure-deriv}}
The proof of \Cref{thm:closure-deriv} relies on the finiteness of derivation trees, which we establish first. 
\begin{lemma}[Finiteness of derivations]\label{lem:deriv-finite}
Every derivation tree in the sense of \Cref{def:derivation} is finite.
\end{lemma}
\begin{proof}[Proof of \Cref{lem:deriv-finite}]
Assume for contradiction that the derivation tree is infinite. By \Cref{def:derivation}(i), every node has finitely many children, since each prerequisite set is finite. Thus the tree is finitely branching. By K\"onig's lemma \cite[Lemma 8.1.2]{ref:diestel_graph_theory}, every infinite finitely branching tree has an infinite descending path. This contradicts the well-foundedness requirement in \Cref{def:derivation}. Therefore the tree is finite.
\end{proof}

\begin{proof}[Proof of \Cref{thm:closure-deriv}]
For ($\Leftarrow$), suppose $\mc K\vdash_\mind c$. By \Cref{lem:deriv-finite}, the derivation tree is finite. We argue by induction on its height. If the height is $0$, then either $c\in \mc K$, or $(\varnothing,c)\in \mathcal E_\mind$; in either case $c\in \cl_\mind(\mc K)$. For the induction step, if the root uses a rule $(\mc S,c)$ and each child label $s\in \mc S$ has a derivation of smaller height, then by the induction hypothesis $\mc S\subseteq \cl_\mind(\mc K)$, hence $c\in \Phi_\mind(\cl_\mind(\mc K))=\cl_\mind(\mc K)$.

For ($\Rightarrow$), let $
\mc D=\{d\in \mc C:\mc K\vdash_\mind d\}$. We show that $\mc D$ is a fixed point of $\Phi_\mind$ containing $\mc K$.
First, $\mc K\subseteq \mc D$: for any $c \in \mc K$, the single-node tree with root labeled $c$ is a valid derivation, so $c \in \mc D$.
Next, we will show that $\Phi_\mind(\mc D) \subseteq \mc D$. Let $c \in \Phi_\mind(\mc D)$. If $c\notin \mc D$, then there exists $(\mc S,c)\in \mathcal E_\mind$ with $\mc S\subseteq \mc D$. For each $s\in \mc S$, choose a derivation of $s$ from $\mc K$ and attach them below a new root labeled $c$. This gives a derivation of $c$, contradiction.
If $\mc S = \varnothing$, the new root has no children and still forms a valid derivation. Thus, $c \in \mc D$, and $\Phi_\mind(\mc D)\subseteq \mc D$. By extensiveness, $\mc D\subseteq \Phi_\mind(\mc D)$, so $\mc D$ is a fixed point containing $\mc K$. Therefore $\cl_\mind(\mc K)\subseteq \mc D$. By definition of $\mc D$, this means that if $c \in \cl_\mind(\mc K)$, then $\mc K \vdash_\mind c$.
\end{proof}

\subsubsection*{Proof of \Cref{thm:alg-equiv}}
We first recall the abstract definition of algebraic closure operator.
\begin{definition}[Algebraic closure operator]\label{def:alg-closure}
Let $\mc X$ be a set. A map $f:2^{\mc X}\to 2^{\mc X}$ is an \emph{algebraic closure operator} if it satisfies extension, monotonicity, idempotence, and the finitary property:
if $c\in f(\mc K)$, then $c\in f(\mc S)$ for some finite $\mc S\subseteq \mc K$.
\end{definition}

\begin{proof}[Proof of \Cref{thm:alg-equiv}]For \textnormal{(i)}, extension, monotonicity, and idempotence of $\cl_\mind$ follow from \Cref{prop:closure-char}. Finitariness follows from \Cref{thm:closure-deriv} and \Cref{lem:deriv-finite}: if $c\in \cl_\mind(\mc K)$, then there is a finite derivation tree using only finitely many base labels from $\mc K$.

For \textnormal{(ii)}, define $
\mathcal E
=
\{(\mc S,c): \mc S\subseteq \mc X \text{ finite and } c\in f(\mc S)\setminus \mc S\}$.
Let $g$ be the closure operator induced by $\mathcal E$ as in the theorem statement. We show $g(\mc K)=f(\mc K)$ for every $\mc K\subseteq \mc X$.

First, $g(\mc K)\subseteq f(\mc K)$ because $f(\mc K)$ is a fixed point of $\Psi_{\mathcal E}$ containing $\mc K$: if $c\in \Psi_{\mathcal E}(f(\mc K))$, then either $c\in f(\mc K)$ or else there exists $(\mc S,c)\in\mathcal E$ with $\mc S\subseteq f(\mc K)$, which implies $c\in f(\mc S)\subseteq f(f(\mc K))=f(\mc K)$.

Conversely, if $c\in f(\mc K)$, then by algebraicity there exists a finite $\mc S_0\subseteq \mc K$ such that $c\in f(\mc S_0)$. If $c\in \mc S_0$, then $c\in \mc K\subseteq g(\mc K)$. Otherwise $(\mc S_0,c)\in \mathcal E$, and since $\mc S_0\subseteq \mc K\subseteq g(\mc K)$, the rule fires inside $g(\mc K)$, so $c\in g(\mc K)$.
\end{proof}

\subsubsection*{Proof of \Cref{thm:ordering}}
If $c^*\in \Acal_\mind$, the empty curriculum works. Assume therefore that
$c^*\notin \Acal_\mind$.
Since $c^*\in \mc U_\mind=\cl_\mind(\Acal_\mind)$, \Cref{thm:closure-deriv}
implies that there exists a derivation tree of $c^*$ from $\Acal_\mind$.
By \Cref{lem:deriv-finite}, this derivation tree is finite.
Let $\mc R$ be the set of all non-base rule nodes in this derivation tree.
Form a directed graph on $\mc R$ by retaining the parent-child relation between
rule nodes and orienting each edge from child to parent. Because the derivation
tree is finite and well-founded, this directed graph is finite and acyclic.
Hence it admits a topological ordering $
v_1,\dots,v_L$.
For each $i=1,\dots,L$, let $(\mc S_i,c_i)$ be the expansion rule attached to
the node $v_i$. Define $
\mc K_0=\Acal_\mind$,
$\mc K_i=\mc K_{i-1}\cup\{c_i\}$ for $i=1,\dots,L$.
We claim that $
\gamma=((\mc S_1,c_1),\dots,(\mc S_L,c_L))$
is a valid ordered curriculum starting from $\Acal_\mind$.
Indeed, fix $i\in\{1,\dots,L\}$ and let $s\in \mc S_i$. In the derivation tree,
the child corresponding to $s$ is either:

(i) a base node, in which case $s\in \Acal_\mind=\mc K_0\subseteq \mc K_{i-1}$; or

(ii) a rule node. In that case this child must occur earlier than $v_i$ in the
topological order, say it is $v_j$ with $j<i$. Its label is then $c_j=s$, so
$s\in \mc K_j\subseteq \mc K_{i-1}$.
Thus every prerequisite in $\mc S_i$ belongs to $\mc K_{i-1}$, so
$\mc S_i\subseteq \mc K_{i-1}$. Since also $(\mc S_i,c_i)\in\mathcal E_\mind$
by construction, each step is valid. Therefore $\gamma$ is a valid ordered
curriculum.
Finally, the root of the derivation tree is a rule node labelled by $c^*$.
Hence it is one of the nodes $v_1,\dots,v_L$, say $v_r$, and therefore
$c_r=c^*$. It follows that $
c^*\in \mc K_r\subseteq \mc K_L$.
So the curriculum reaches a final acquired concept set containing $c^*$.

\subsubsection*{Proof of \Cref{prop:unteachable}}
We argue by induction on $i$.
For $i=0$, $
\mc K_0=\Acal_\mind\subseteq \mc U_\mind$
by \Cref{prop:U-unique}(i).
Now suppose $\mc K_{i-1}\subseteq \mc U_\mind$. Since $\gamma$ is valid,
$(\mc S_i,c_i)\in\mathcal E_\mind$ and $
\mc S_i\subseteq \mc K_{i-1}\subseteq \mc U_\mind$.
By \Cref{prop:U-unique}(ii), this implies $
c_i\in \mc U_\mind$.
Hence $
\mc K_i=\mc K_{i-1}\cup\{c_i\}\subseteq \mc U_\mind$.
This proves the claim for all $i$.
The final statement follows immediately.

\subsubsection*{Proof of \Cref{prop:reach-structure}}
Because $\mc U_\mind$ is finite by \Cref{ass:finite}, the power set $2^{\mc U_\mind}$ is finite. Since $
\mathbb K_\mind\subseteq 2^{\mc U_\mind}$,
it follows that $\mathbb K_\mind$ is finite.

For \textnormal{(i)}, the trivial chain of length zero shows that $
\Acal_\mind\in\mathbb K_\mind$.
Moreover, every reachable set contains $\Acal_\mind$, since every witnessing chain starts from $\Acal_\mind$ and only adds concepts. Thus $\Acal_\mind$ is the minimum element of $(\mathbb K_\mind,\subseteq)$.

For \textnormal{(ii)}, let $\mc K\in\mathbb K_\mind$ with $\mc K\neq \Acal_\mind$. By definition of reachability, there exists a witnessing chain
\[
\Acal_\mind=\mc K_0\subset \mc K_1\subset\cdots\subset \mc K_L=\mc K
\]
such that
\[
\mc K_{i+1}=\mc K_i\cup\{c_i\},
\qquad
c_i\in \Phi_\mind(\mc K_i)\setminus \mc K_i
\qquad (i=0,\dots,L-1).
\]
Since $\mc K\neq \Acal_\mind$, one has $L\ge 1$. Then $\mc K_{L-1}\in\mathbb K_\mind$ by \Cref{lem:prefix},
\[
\mc K_{L-1}\subset \mc K,
\qquad
|\mc K\setminus \mc K_{L-1}|=1.
\]
This proves \textnormal{(ii)}.

For \textnormal{(iii)}, let $\mc K,\mc K'\in\mathbb K_\mind$. Choose a witnessing chain for $\mc K'$:
\[
\Acal_\mind=\mc K_0'\subset \mc K_1'\subset\cdots\subset \mc K_s'=\mc K',
\qquad
\mc K_{i+1}'=\mc K_i'\cup\{c_i\},
\qquad
c_i\in \Phi_\mind(\mc K_i')\setminus \mc K_i'.
\]
For each $i=0,\dots,s$, define
\[
\mc L_i=\mc K\cup \mc K_i'.
\]
Then $\mc L_0=\mc K$ and $\mc L_s=\mc K\cup\mc K'$. Since $\mc K_i'\subseteq \mc L_i$, monotonicity of $\Phi_\mind$ gives
\[
\Phi_\mind(\mc K_i')\subseteq \Phi_\mind(\mc L_i).
\]
Hence, whenever $c_i\notin \mc L_i$, one has $
c_i\in \Phi_\mind(\mc K_i')\subseteq \Phi_\mind(\mc L_i)$,
so
\[
\mc L_{i+1}=\mc L_i\cup\{c_i\}
\]
is a valid extension. If instead $c_i\in \mc L_i$, then $\mc L_{i+1}=\mc L_i$.

Removing repeated sets from the sequence $(\mc L_i)_{i=0}^s$ yields a valid chain from $\mc K$ to $\mc K\cup\mc K'$. Concatenating this chain with any witnessing chain from $\Acal_\mind$ to $\mc K$ shows that $
\mc K\cup\mc K'\in\mathbb K_\mind$.
Thus $\mathbb K_\mind$ is union-closed.

For \textnormal{(iv)}, we first show that $\mc U_\mind\in\mathbb K_\mind$. Let $\mc K\in\mathbb K_\mind$ with $\mc K\neq \mc U_\mind$. Suppose, toward a contradiction, that $
\Phi_\mind(\mc K)=\mc K$.
Then $\mc K$ is a fixed point of $\Phi_\mind$ containing $\Acal_\mind$. Since $
\mc U_\mind=\cl_\mind(\Acal_\mind)$
is the least fixed point containing $\Acal_\mind$, it follows that $
\mc U_\mind\subseteq \mc K$.
But by definition of $\mathbb K_\mind$, one also has $\mc K\subseteq \mc U_\mind$, hence $\mc K=\mc U_\mind$, a contradiction. Therefore $
\Phi_\mind(\mc K)\setminus \mc K\neq \varnothing$.
Choose any $
c\in \Phi_\mind(\mc K)\setminus \mc K$.
Because $\mc K\subseteq \mc U_\mind$ and $\mc U_\mind$ is a fixed point of $\Phi_\mind$, monotonicity gives $
\Phi_\mind(\mc K)\subseteq \Phi_\mind(\mc U_\mind)=\mc U_\mind$,
so in particular $c\in \mc U_\mind$. Hence $\mc K\cup\{c\}$ is again a reachable subset of $\mc U_\mind$.

Starting from $\Acal_\mind$, repeat this step as long as $\Phi_\mind(\mc K) \backslash \mc K \neq \varnothing$. Because $\mc U_\mind$ is finite and each step strictly enlarges the set, the process terminates after finitely many steps at some reachable set $\mc F\subseteq \mc U_\mind$ satisfying $ 
\Phi_\mind(\mc F)=\mc F$.
Since $\mc F$ is a fixed point containing $\Acal_\mind$, minimality of $\mc U_\mind=\cl_\mind(\Acal_\mind)$ implies $
\mc U_\mind\subseteq \mc F$.
As also $\mc F\subseteq \mc U_\mind$, we conclude that $\mc F=\mc U_\mind$. Therefore $
\mc U_\mind\in \mathbb K_\mind$.
Since every element of $\mathbb K_\mind$ is by definition a subset of $\mc U_\mind$, it follows that $\mc U_\mind$ is the maximum element of $(\mathbb K_\mind,\subseteq)$.

Finally, \textnormal{(v)} follows from \textnormal{(iii)}. For any $\mc K,\mc K'\in\mathbb K_\mind$, the set $\mc K\cup\mc K'$ belongs to $\mathbb K_\mind$ and is clearly an upper bound of $\mc K$ and $\mc K'$. If $\mc M\in\mathbb K_\mind$ is any other upper bound, so that
\[
\mc K\subseteq \mc M
\qquad\text{and}\qquad
\mc K'\subseteq \mc M,
\]
then $
\mc K\cup\mc K'\subseteq \mc M$.
Hence $\mc K\cup\mc K'$ is the least upper bound.
\qed
\subsubsection*{Proof of \Cref{cor:reachable-learning-space}}
The result follows directly from \Cref{prop:reach-structure} and \Cref{def:A-based-learning-space}.
\qed
\subsubsection*{Proof of Theorem~\ref{thm:representation-learning-space}}
We prove \textnormal{(ii)}$\Rightarrow$\textnormal{(i)} and \textnormal{(i)}$\Rightarrow$\textnormal{(ii)} separately.

\medskip
\noindent\textbf{\textnormal{(ii)}$\Rightarrow$\textnormal{(i)}.}
Assume there exists a mind $
\mind=(\mc C,\Acal,\mathcal E_\mind)$
such that $
\mathbb K_\mind=\mathbb F$.
By \Cref{cor:reachable-learning-space}, the family $\mathbb K_\mind$ is an $\Acal$-based learning space. Hence so is $\mathbb F$.

\medskip
\noindent\textbf{\textnormal{(i)}$\Rightarrow$\textnormal{(ii)}.}
Assume that $\mathbb F$ is an $\Acal$-based learning space, and define $
\mind_{\mathbb F}=(\mc C,\Acal,\mathcal E_{\mathbb F})$
using the canonical rule set above. Let $\mathbb K_{\mind_{\mathbb F}}$ denote the reachable family generated by this mind. We prove that $\mathbb K_{\mind_{\mathbb F}}=\mathbb F$.
\noindent\emph{Step 1: $\mathbb K_{\mind_{\mathbb F}}\subseteq \mathbb F$.}
Let $\mc K\in \mathbb K_{\mind_{\mathbb F}}$. Choose a witnessing chain $
\Acal=\mc K_0\subset \mc K_1\subset \cdots\subset \mc K_L=\mc K$
such that
\[
\mc K_{i+1}=\mc K_i\cup\{c_i\},
\qquad
c_i\in \Phi_{\mind_{\mathbb F}}(\mc K_i)\setminus \mc K_i
\quad\text{for }i=0,\dots,L-1.
\]
We prove by induction on $i$ that $\mc K_i\in\mathbb F$ for all $i$.

For $i=0$, one has $\mc K_0=\Acal\in\mathbb F$.

Now suppose $\mc K_i\in\mathbb F$. Since $
c_i\in \Phi_{\mind_{\mathbb F}}(\mc K_i)\setminus \mc K_i$,
there exists a rule $(\mc S,c_i)\in\mathcal E_{\mathbb F}$ with $\mc S\subseteq \mc K_i$. By definition of $\mathcal E_{\mathbb F}$, 
$\mc S\in\mathbb F$,
$\mc S\cup\{c_i\}\in\mathbb F$.
Because $\mc K_i\in\mathbb F$ and $\mathbb F$ is union-closed,
\[
\mc K_i\cup(\mc S\cup\{c_i\})\in\mathbb F.
\]
Since $\mc S\subseteq \mc K_i$, this simplifies to
\[
\mc K_i\cup\{c_i\}=\mc K_{i+1}\in\mathbb F.
\]
Thus every $\mc K_i$ lies in $\mathbb F$, and in particular $\mc K\in\mathbb F$. Hence $
\mathbb K_{\mind_{\mathbb F}}\subseteq \mathbb F$.
\noindent\emph{Step 2: $\mathbb F\subseteq \mathbb K_{\mind_{\mathbb F}}$.}
Let $\mc K\in\mathbb F$. If $\mc K=\Acal$, then $\mc K$ is reachable by the trivial chain.

Assume now that $\mc K\neq \Acal$. Since $\mathbb F$ is an $\Acal$-based learning space, repeated application of accessibility yields a descending chain
\[
\mc K=\mc K_L \supset \mc K_{L-1}\supset \cdots \supset \mc K_0=\Acal
\]
such that each $\mc K_i\in\mathbb F$ and
\[
\mc K_i=\mc K_{i-1}\cup\{x_i\}
\qquad\text{for }i=1,\dots,L.
\]
Reverse the chain:
\[
\Acal=\mc K_0\subset \mc K_1\subset \cdots\subset \mc K_L=\mc K.
\]
For each $i=1,\dots,L$, both $\mc K_{i-1}$ and $\mc K_i=\mc K_{i-1}\cup\{x_i\}$ belong to $\mathbb F$. Therefore, by the definition of $\mathcal E_{\mathbb F}$, $
(\mc K_{i-1},x_i)\in \mathcal E_{\mathbb F}$.
Hence $
x_i\in \Phi_{\mind_{\mathbb F}}(\mc K_{i-1})\setminus \mc K_{i-1}$,
so every step in the chain is a valid reachable extension. Thus $\mc K$ is reachable from $\Acal$, which shows that $
\mc K\in \mathbb K_{\mind_{\mathbb F}}$.
Therefore $
\mathbb F\subseteq \mathbb K_{\mind_{\mathbb F}}$.

Combining the two inclusions gives $
\mathbb K_{\mind_{\mathbb F}}=\mathbb F$.
\qed
\subsection{Proofs for Section~\ref{sec:teaching}}\label{app:proofs-teaching}
\subsubsection*{Proof of \Cref{lem:Kt-reachable}}

We argue by induction on $t$.

For $t=0$, $
\mc K_0=\Acal_\mind\in \mathbb K_\mind$
by the trivial witnessing chain.
Now suppose $\mc K_t\in\mathbb K_\mind$ almost surely.
If $Y_{t+1}=\bot$, then by \Cref{def:knowledge-update}, $
\mc K_{t+1}=\mc K_t$,
hence $\mc K_{t+1}\in\mathbb K_\mind$.
If $Y_{t+1}\in\mc Z$, define $
c_{t+1}=\tgt(Y_{t+1})$.
Since the parser outputs a non-null signal, \Cref{def:parsing} implies that $
c_{t+1}\in \Phi_\mind(\mc K_t)$.
Hence either $c_{t+1}\in \mc K_t$, in which case $
\mc K_{t+1}=\mc K_t$,
or else $
c_{t+1}\in \Phi_\mind(\mc K_t)\setminus \mc K_t$,
in which case $
\mc K_{t+1}=\mc K_t\cup\{c_{t+1}\}$
is a valid one-step reachable extension from $\mc K_t$ in the sense of
\Cref{def:reachable}. Since $\mc K_t\in\mathbb K_\mind$, it follows that $
\mc K_{t+1}\in\mathbb K_\mind$.
This proves the claim.
\qed

\subsubsection*{Proof of \Cref{prop:first-law}}
By the definition of conditional mutual information,
\[
\mathds I(\Theta;Y_{t+1}\mid \Fcal_t)
=
\mathds H(\Theta\mid \Fcal_t)
-
\mathbb E[\mathds H(\Theta\mid \Fcal_t\vee \sigma(Y_{t+1}))\mid \Fcal_t].
\]
Since $\Fcal_{t+1}=\Fcal_t\vee \sigma(Y_{t+1})$ and $H_t=\mathds H(\Theta\mid \Fcal_t)$, this becomes $
\mathds I(\Theta;Y_{t+1}\mid \Fcal_t)
=
H_t-\mathbb E[H_{t+1}\mid \Fcal_t]$.
Because $H_t$ is $\Fcal_t$-measurable, $
H_t-\mathbb E[H_{t+1}\mid \Fcal_t]
=
\mathbb E[H_t-H_{t+1}\mid \Fcal_t]$.
\qed

\subsubsection*{Proof of \Cref{thm:arrow}}
By \Cref{prop:first-law},
\[
H_t-\mathbb E[H_{t+1}\mid \Fcal_t]
=
\mathds I(\Theta;Y_{t+1}\mid \Fcal_t)\ge 0.
\]
Hence $(H_t)$ is a supermartingale. Equality holds if and only if
\[
\mathds I(\Theta;Y_{t+1}\mid \Fcal_t)=0,
\]
which is equivalent to conditional independence of $\Theta$ and $Y_{t+1}$ given $\Fcal_t$.
\qed

\subsubsection*{Proof of \Cref{thm:relativity}}
The proof of \Cref{thm:relativity} relies on the following lemma, which we
establish first.
\begin{lemma}[Unparseability erases information]\label{lem:unparseability-erasure}
Let
\[
C_{t+1}=\tgt(Z_{t+1}),
\qquad
U_{t+1}=\{C_{t+1}\notin \Phi_\mind(\mc K_t)\}.
\]
Then:
\begin{enumerate}[nosep,label=\textnormal{(\roman*)}]
\item on $U_{t+1}$ one has $Y_{t+1}=\bot$ almost surely, and therefore
\[
\mathds I(\Theta;Y_{t+1}\mid \Fcal_t,U_{t+1})=0,
\qquad
\mathds I(Z_{t+1};Y_{t+1}\mid \Fcal_t,U_{t+1})=0;
\]
\item if $\mathbb P(U_{t+1}\mid \Fcal_t)=1$, then
\[
\mathds I(\Theta;Y_{t+1}\mid \Fcal_t)=0,
\qquad
\mathds I(Z_{t+1};Y_{t+1}\mid \Fcal_t)=0.
\]
\end{enumerate}
\end{lemma}

\begin{proof}
On $U_{t+1}$, \Cref{def:parsing} gives
\[
Y_{t+1}=\rho_\mind(Z_{t+1},\mc K_t)=\bot
\qquad\text{almost surely.}
\]
Hence conditional on $(\Fcal_t,U_{t+1})$, the random variable $Y_{t+1}$ is constant, so all the relevant conditional entropies are zero. This proves \textnormal{(i)}.

If $\mathbb P(U_{t+1}\mid \Fcal_t)=1$, then $U_{t+1}$ occurs almost surely conditional on $\Fcal_t$, so $Y_{t+1}=\bot$ almost surely conditional on $\Fcal_t$. Again all relevant conditional entropies are zero, proving \textnormal{(ii)}.
\end{proof}

\begin{proof}[Proof of \Cref{thm:relativity}]On $U_{t+1}$, \Cref{lem:unparseability-erasure} gives $
\mathds I(\Theta;Y_{t+1}\mid \Fcal_t,U_{t+1})=0$.
On $U_{t+1}^c$, \Cref{prop:parseable-positive} yields $
\mathds I(\Theta;Y_{t+1}\mid \Fcal_t,U_{t+1}^c)
=
\mathds I(\Theta;Z_{t+1}\mid \Fcal_t,U_{t+1}^c)>0$.
\end{proof}

\subsection{Proofs for Section~\ref{sec:information-transfer-speed-teaching}}\label{app:proofs-speed}

\begin{lemma}[Explicit formula for the parsed entropy bound]
\label{lem:cap-formula}
Assume $\mc Z$ is finite. Then, for every $\mc K\in\mathbb K_\mind$,
\[
C_\mind(\mc K)=
\begin{cases}
\log\bigl(|\mathcal Z_{\mathrm{ord}}(\mc K)|+1\bigr),
& \text{if }\mathcal Z_{\mathrm{ord}}(\mc K)\subsetneq \mc Z,\\[0.4em]
\log |\mc Z|,
& \text{if }\mathcal Z_{\mathrm{ord}}(\mc K)=\mc Z.
\end{cases}
\]
\end{lemma}
\begin{proof}[Proof of \Cref{lem:cap-formula}]
Fix $\mc K\in\mathbb K_\mind$ and define the parsed observation range
\[
\mc Y(\mc K)=\{\rho_\mind(z,\mc K):z\in\mc Z\}\subseteq \mc Z\cup\{\bot\}.
\]
For any $\mc Z$-valued random variable $Z$, the random variable
$\rho_\mind(Z,\mc K)$ takes values in $\mc Y(\mc K)$ almost surely, so $
\mathds H(\rho_\mind(Z,\mc K))\le \log |\mc Y(\mc K)|$ by \cite[Page 41]{cover}.
Taking the supremum over all such $Z$ gives $
C_\mind(\mc K)\le \log |\mc Y(\mc K)|$.

For the reverse inequality, let $M=|\mc Y(\mc K)|$. For each
$y\in \mc Y(\mc K)$, choose some representative $z_y\in\mc Z$ such that $\rho_\mind(z_y,\mc K)=y$.
Define a $\mc Z$-valued random variable $Z$ by
\[
\mathbb P(Z=z_y)=\frac{1}{M}
\qquad\text{for each } y\in \mc Y(\mc K),
\]
and $\mathbb P(Z=z)=0$ for all other $z\in\mc Z$. Then
$\rho_\mind(Z,\mc K)$ is uniform on $\mc Y(\mc K)$, so
\[
\mathds H(\rho_\mind(Z,\mc K))=\log |\mc Y(\mc K)|.
\]
Hence $
C_\mind(\mc K)=\log |\mc Y(\mc K)|$.

Under the parsing map $\rho_\mind$, one has
\[
\rho_\mind(z,\mc K)=
\begin{cases}
z, & \text{if } z\in \mathcal Z_{\mathrm{ord}}(\mc K),\\
\bot, & \text{if } z\notin \mathcal Z_{\mathrm{ord}}(\mc K).
\end{cases}
\]
If $\mathcal Z_{\mathrm{ord}}(\mc K)\subsetneq \mc Z$, then
$
\mc Y(\mc K)=\mathcal Z_{\mathrm{ord}}(\mc K)\cup\{\bot\}$,
so $
|\mc Y(\mc K)|=|\mathcal Z_{\mathrm{ord}}(\mc K)|+1$.
If instead $\mathcal Z_{\mathrm{ord}}(\mc K)=\mc Z$, then every raw signal is
parseable, so $
\mc Y(\mc K)=\mc Z$.
Substituting these two cases into $
C_\mind(\mc K)=\log |\mc Y(\mc K)|$
proves the claim.
\end{proof}

\subsubsection*{Proof of \Cref{prop:state-info-bound}}
Because $\mc K_t$ is $\Fcal_t$-measurable, conditional on $\Fcal_t$ the law of $
Y_{t+1}=\rho_\mind(Z_{t+1},\mc K_t)$
is obtained by passing the conditional law of $Z_{t+1}$ through the fixed map $
z\mapsto \rho_\mind(z,\mc K_t)$.
Therefore,
\[
\mathds I(\Theta;Y_{t+1}\mid \Fcal_t)
\le
\mathds H(Y_{t+1}\mid \Fcal_t)
\le
C_\mind(\mc K_t)
\qquad\text{almost surely.}
\]
\qed
\subsubsection*{Proof of \Cref{lem:cap-mono}}
Assume $\mc K,\mc K'\in\mathbb K_\mind$ with $\mc K\subseteq \mc K'$. By
\Cref{lem:mono}, $
\Phi_\mind(\mc K)\subseteq \Phi_\mind(\mc K')$,
hence $
\mathcal Z_{\mathrm{ord}}(\mc K)\subseteq \mathcal Z_{\mathrm{ord}}(\mc K')$.

Define $
g_{\mc K,\mc K'}:\mc Z\cup\{\bot\}\to \mc Z\cup\{\bot\}$
by
\[
g_{\mc K,\mc K'}(y)=
\begin{cases}
y, & \text{if } y\in \mathcal Z_{\mathrm{ord}}(\mc K),\\
\bot, & \text{otherwise.}
\end{cases}
\]
Then for every $z\in\mc Z$, $
\rho_\mind(z,\mc K)=g_{\mc K,\mc K'}(\rho_\mind(z,\mc K'))$.
Indeed, if $z\in \mathcal Z_{\mathrm{ord}}(\mc K)$, then $z$ is ordered at both
sets and both sides equal $z$. If $z\notin \mathcal Z_{\mathrm{ord}}(\mc K)$,
then the left-hand side is $\bot$; on the right-hand side, either
$\rho_\mind(z,\mc K')=\bot$, or else $\rho_\mind(z,\mc K')=z$ and
$g_{\mc K,\mc K'}(z)=\bot$.
Now let $Z$ be any $\mc Z$-valued random variable. Then
\[
\rho_\mind(Z,\mc K)
=
g_{\mc K,\mc K'}\bigl(\rho_\mind(Z,\mc K')\bigr)
\qquad\text{almost surely.}
\]
Thus $\rho_\mind(Z,\mc K)$ is a deterministic function of
$\rho_\mind(Z,\mc K')$. By the data processing inequality, $
\mathds H(\rho_\mind(Z,\mc K))
\le
\mathds H(\rho_\mind(Z,\mc K'))$.
Taking suprema over all $\mc Z$-valued random variables $Z$ yields $
C_\mind(\mc K)\le C_\mind(\mc K')$.
\qed
\subsubsection*{Proof of Theorem~\ref{thm:blackwell}}
Let $g_{\mc K,\mc K'}$ be the deterministic map constructed in the proof of
\Cref{lem:cap-mono}. For every raw signal $z\in\mc Z$, one has $
\rho_\mind(z,\mc K)
=
g_{\mc K,\mc K'}(\rho_\mind(z,\mc K'))$.
Therefore, conditional on the public history $h_t$,
\[
\rho_\mind(Z_{t+1},\mc K)
=
g_{\mc K,\mc K'}\bigl(\rho_\mind(Z_{t+1},\mc K')\bigr)
\qquad\text{almost surely.}
\]

Hence, for every $\omega\in\Omega$ and every $y\in\mc Z\cup\{\bot\}$,
\[
\mathds W_{\mc K,h_t}(y\mid \omega)
=
\sum_{y'\in \mc Z\cup\{\bot\}}
G_{\mc K,\mc K'}(y\mid y')\,
\mathds W_{\mc K',h_t}(y'\mid \omega),
\]
where $
G_{\mc K,\mc K'}(y\mid y')
=
\mathbf 1\{g_{\mc K,\mc K'}(y')=y\}$
is the Markov kernel induced by $g_{\mc K,\mc K'}$.

Thus $\mathds W_{\mc K,h_t}$ is obtained from $\mathds W_{\mc K',h_t}$ by
post-processing through a Markov kernel independent of $\omega$. Therefore
$\mathds W_{\mc K,h_t}$ is a garbling of $\mathds W_{\mc K',h_t}$, and
$\mathds W_{\mc K',h_t}$ Blackwell-dominates $\mathds W_{\mc K,h_t}$.
\qed

\subsubsection*{Proof of \Cref{prop:structural-barrier}}
Fix a sample path. By the concept-acquisition update rule, each round adds at
most one new concept to the learner's acquired concept set. If completion occurs
at time $\tau$, then in particular $
\Theta\in \mc K_\tau$.
Delete repeated sets from the sequence $
\mc K_0,\mc K_1,\dots,\mc K_\tau$.
The resulting strictly increasing sequence is of the form
\[
\Acal_\mind=\mc K_{i_0}\subset \mc K_{i_1}\subset \cdots \subset \mc K_{i_r},
\]
where each step adds one concept belonging to the one-step expansion of
the previous set. Hence it is a prerequisite-respecting chain ending at a set
containing $\Theta$.

By definition of $L_\mind(\Theta)$, any such chain has length at least
$L_\mind(\Theta)$. Since the number of strict acquisitions up to time $\tau$ is
at most $\tau$, it follows that $
\tau\ge L_\mind(\Theta)$ \text{almost surely.}
Taking expectations completes the proof.
\qed

\subsubsection*{Proof of \Cref{lem:chain}}
By \Cref{prop:first-law}, $
\mathbb E[H_t-H_{t+1}\mid \Fcal_t]
=
\mathds I(\Theta;Y_{t+1}\mid \Fcal_t)$.
Multiplying by $\mathbf 1_{\{\tau_{\mathrm{id}}>t\}}$ and taking expectations
gives
\[
\mathbb E\left[\mathbf 1_{\{\tau_{\mathrm{id}}>t\}}(H_t-H_{t+1})\right]
=
\mathbb E\left[\mathbf 1_{\{\tau_{\mathrm{id}}>t\}}
\mathds I(\Theta;Y_{t+1}\mid \Fcal_t)\right].
\]
Summing from $t=0$ to $n-1$ yields
\[
\mathbb E\left[\sum_{t=0}^{n-1}\mathbf 1_{\{\tau_{\mathrm{id}}>t\}}
\mathds I(\Theta;Y_{t+1}\mid \Fcal_t)\right]
=
\mathbb E[H_0]-\mathbb E[H_{\tau_{\mathrm{id}}\wedge n}].
\]
Since $\Theta$ is $\Fcal_{\tau_{\mathrm{id}}}$-measurable, one has
\[
H_{\tau_{\mathrm{id}}}
=
\mathds H(\Theta\mid \Fcal_{\tau_{\mathrm{id}}})
=
0
\qquad\text{almost surely.}
\]
Also, $
0\le H_t\le \log |\Omega|$ \text{for all } $t$.
Let
\[
S_n=
\sum_{t=0}^{n-1}\mathbf 1_{\{\tau_{\mathrm{id}}>t\}}
\mathds I(\Theta;Y_{t+1}\mid \Fcal_t).
\]
Because the summands are nonnegative, $S_n$ increases almost surely to
\[
\sum_{t=0}^{\tau_{\mathrm{id}}-1}
\mathds I(\Theta;Y_{t+1}\mid \Fcal_t).
\]
Monotone convergence and bounded convergence therefore give
\[
\mathbb E\left[\sum_{t=0}^{\tau_{\mathrm{id}}-1}
\mathds I(\Theta;Y_{t+1}\mid \Fcal_t)\right]
=
\mathbb E[H_0].
\]
Since $\Fcal_0$ is trivial, $
\mathbb E[H_0]=\mathds H(\Theta)$.
\qed
\subsubsection*{Proof of \Cref{prop:trajectory-budget}}
By \Cref{lem:chain},
\[
\mathds H(\Theta)
=
\mathbb E\left[\sum_{t=0}^{\tau_{\mathrm{id}}-1}
\mathds I(\Theta;Y_{t+1}\mid \Fcal_t)\right].
\]
By \Cref{prop:state-info-bound},
\[
\mathds I(\Theta;Y_{t+1}\mid \Fcal_t)\le C_\mind(\mc K_t)
\qquad\text{almost surely for every } t.
\]
Substituting this bound inside the sum yields the result.
\qed
\subsubsection*{Proof of \Cref{thm:struct-info-bound}}
By \Cref{prop:structural-barrier}, the structural bound follows: $
\mathbb E[\tau]\ge \mathbb E[L_\mind(\Theta)]$.
For the epistemic part: define the identification time
\[
\tau_{\mathrm{id}} = \inf \{ t \ge 0 : \mathds H(\Theta \mid \mathcal{F}_t) = 0 \}.
\]

Since $\{\tau_{\mathrm{id}} \le t\} = \{H(\Theta \mid \mathcal{F}_t) = 0\} \in \mathcal{F}_t$,
$\tau_{\mathrm{id}}$ is an $(\mathcal{F}_t)$-stopping time.

Moreover, if $\tau$ is a completion time then identification must already
have occurred, so $\tau_{\mathrm{id}} \le \tau$ almost surely.
By \Cref{prop:trajectory-budget},
\[
\mathds H(\Theta)
\le
\mathbb E\left[\sum_{t=0}^{\tau_{\mathrm{id}}-1} C_\mind(\mc K_t)\right].
\]
{Since $\mc K_t\in \mathbb K_\mind$ almost surely by \Cref{lem:Kt-reachable},}
\[
C_\mind(\mc K_t)\le C_\mind^{\max}
\qquad\text{almost surely.}
\]
Therefore
\[
\mathds H(\Theta)
\le \EE\left[\sum\limits_{t=0}^{\tau_{\mathrm{id}} - 1} C_\mind(\mc K_t)\right] \leq
C_\mind^{\max}\,\mathbb E[\tau_{\mathrm{id}}]
\le
C_\mind^{\max}\,\mathbb E[\tau].
\]
Rearranging yields
\[
\mathbb E[\tau]\ge \frac{\mathds H(\Theta)}{C_\mind^{\max}}.
\]
Combining the two bounds gives the theorem.
\qed

\subsubsection*{Proof of \Cref{prop:direct-target-collapse}}
For each $c\in\Omega_+$, choose one raw signal $z_c\in\mc Z$ such that $
\tgt(z_c)=c$.
Because $\tgt$ is a function, the signals $(z_c)_{c\in\Omega_+}$ are pairwise
distinct. Since $\Omega\subseteq \mc U_\mind$ and $\mc U_\mind$ is a fixed
point of $\Phi_\mind$, every target concept $c\in\Omega$ is ordered at
$\mc U_\mind$. Hence each $z_c$ is parseable at $\mc U_\mind$, so the parsed
observation range at $\mc U_\mind$ contains at least the distinct symbols $
\{z_c:c\in\Omega_+\}$.
Therefore, by \Cref{lem:cap-formula}, $
C_\mind^{\max}\ge C_\mind(\mc U_\mind)\ge \log |\Omega_+|$.
Since entropy is bounded by the logarithm of the support size, $
\mathds H(\Theta)\le \log |\Omega_+|$,
which yields $
{\mathds H(\Theta)} \leq {C_\mind^{\max}}$.
This proves \textnormal{(i)}.

For \textnormal{(ii)}, fix $c\in\Omega_+$. By definition of $L_\mind(c)$,
there exists a valid ordered curriculum of length $L_\mind(c)$ from
$\Acal_\mind$ to a set containing $c$. Under \Cref{ass:structural-signals}, the
teacher can implement that curriculum by sending one raw signal targeting each
concept along the path. After $L_\mind(c)$ rounds, the learner has
acquired $c$.

In one additional round, the teacher sends the fixed representative signal
$z_c$. Because $c\in \mc K_{L_\mind(c)}$, the signal $z_c$ is parseable at that
state. Since the strategy specifies a unique representative signal for each
possible target, the learner identifies the realized target after observing
$z_c$. Thus
\[
\tau\le L_\mind(\Theta)+1
\qquad\text{almost surely.}
\]
Taking expectations completes the proof.
\qed

\subsection{Proofs for \Cref{sec:structural-limits}}
\label{app:proofs-impossibility}
\subsubsection*{Proof of Proposition~\ref{prop:hard-zero}}
By \Cref{prop:structural-barrier}, every completion time $\tau$ satisfies
\[
\tau\ge L_\mind(\Theta)
\qquad\text{almost surely.}
\]
Hence $
\{\tau\le t\}\subseteq \{L_\mind(\Theta)\le t\}$.
Therefore, for every admissible teaching strategy,
\[
\mathbb P(\tau\le t)\le \mathbb P\bigl(L_\mind(\Theta)\le t\bigr).
\]
Taking the supremum over strategies yields the first claim.

If $t<L_{\min}$, then $L_\mind(\Theta)>t$ almost surely under the prior, so
\[
\mathbb P\bigl(L_\mind(\Theta)\le t\bigr)=0.
\]
Hence $\mathds V(t)=0$.
\qed

\subsubsection*{Proof of \Cref{prop:eventual-success}}
Fix an admissible teaching strategy with $\mathbb E[\tau]<\infty$. By Markov's
inequality,
\[
\mathbb P(\tau>t)\le \frac{\mathbb E[\tau]}{t},
\]
and therefore
\[
\mathbb P(\tau\le t)\ge 1-\frac{\mathbb E[\tau]}{t}.
\]
Since $\mathds V(t)$ is the supremum of $\mathbb P(\tau\le t)$ over all
admissible strategies, it follows that
\[
\mathds V(t)\ge 1-\frac{\mathbb E[\tau]}{t}.
\]
Letting $t\to\infty$ gives $\mathds V(t)\to 1$.
\qed

\subsubsection*{Proof of \Cref{prop:deterministic-step}}
By \Cref{prop:structural-barrier}, every acquisition time $\tau_g$ satisfies
\[
\tau_g\ge L_\mind(g)
\qquad\text{almost surely.}
\]
Therefore, for every admissible strategy and every $t<L_\mind(g)$, $
\mathbb P(\tau_g\le t)=0$.
Taking the supremum over strategies yields
$
\mathds V_g(t)=0$ for $t<L_\mind(g)$.

Now let $L=L_\mind(g)$. By definition of structural distance, there exists a
witnessing chain
\[
\Acal_\mind=\mc K_0\subset \mc K_1\subset \cdots \subset \mc K_L,
\qquad
g\in \mc K_L,
\]
such that
\[
\mc K_{i+1}=\mc K_i\cup\{u_i\},
\qquad
u_i\in \Phi_\mind(\mc K_i)\setminus \mc K_i
\quad (i=0,\dots,L-1).
\]
By \Cref{ass:structural-signals}, for each $u_i$ there exists a raw signal
$z_i\in\mc Z$ such that
\[
\tgt(z_i)=u_i.
\]
Since $u_i\in \Phi_\mind(\mc K_i)$, the signal $z_i$ is parseable at $\mc K_i$.
If the teacher sends
\[
z_0,z_1,\dots,z_{L-1}
\]
in sequence, the learner moves through the sets
\[
\mc K_0,\mc K_1,\dots,\mc K_L
\]
and therefore acquires $g$ after $L$ rounds. Thus there exists an
admissible strategy such that
\[
\mathbb P(\tau_g\le L)=1.
\]
Hence
\[
\mathds V_g(t)=1
\qquad\text{for all } t\ge L_\mind(g).
\]
\qed

\subsubsection*{Proof of \Cref{prop:capital-destruction}}
For \textnormal{(i)}, if every learner receives fewer than $L$ rounds, then by
\Cref{prop:deterministic-step} the acquisition probability of $g$ is zero for
every learner. Hence no learner completes.

For \textnormal{(ii)}, select $\min\{N,\lfloor B/L\rfloor\}$ learners and allocate
$L$ rounds to each of them, allocating $0$ rounds to the remaining learners.
This is feasible because $
L \min\{N,\lfloor B/L\rfloor\} \le L\lfloor B/L\rfloor \le B$,
and \(\min\{N,\lfloor B/L\rfloor\} \le N\).
By \Cref{prop:deterministic-step}, each selected learner completes with
probability one. The remaining learners receive \(0<L\) rounds, so again by
\Cref{prop:deterministic-step} they complete with probability zero. Hence the
total number of completed learners is $
\min\{N,\lfloor B/L\rfloor\}$.
\qed
\subsubsection*{Proof of \Cref{thm:impossibility}}
Let $a$ and $g$ be distinct concepts. For each $i\in\{1,\dots,k\}$ and each
$j\in\{1,\dots,L-1\}$, let $p_{i,j}$ be a distinct concept, with all these
concepts also distinct from $a$ and $g$. Define
\[
\mc C
=
\{a,g\}
\cup
\{p_{i,j}: i=1,\dots,k,\ j=1,\dots,L-1\},
\qquad
\Acal=\{a\}.
\]

For each $i\in\{1,\dots,k\}$, define the rule set of mind $\mind_i$ by
\[
\mathcal E_{\mind_i}
=
\bigl\{
(\{a\},p_{i,1}),
(\{p_{i,1}\},p_{i,2}),
\dots,
(\{p_{i,L-2}\},p_{i,L-1}),
(\{p_{i,L-1}\},g)
\bigr\}.
\]
Thus mind $\mind_i$ has a private prerequisite chain
\[
a \to p_{i,1}\to p_{i,2}\to \cdots \to p_{i,L-1}\to g,
\]
and no concept $p_{i',j}$ with $i'\neq i$ is reachable in mind $\mind_i$.

Choose raw signals
\[
z_{i,j}\in \mc Z
\qquad (i=1,\dots,k,\ j=1,\dots,L-1),
\qquad
z_g\in \mc Z,
\]
with
\[
\tgt(z_{i,j})=p_{i,j},
\qquad
\tgt(z_g)=g.
\]

\medskip
\noindent\textit{(i) Personalized acquisition in $L$ rounds.}
Fix $i$. The sequence
\[
z_{i,1},z_{i,2},\dots,z_{i,L-1},z_g
\]
is a valid ordered curriculum of length $L$ for $\mind_i$: each signal
becomes parseable when its predecessor on the private chain has been
acquired, and the final signal acquires $g$.

\medskip
\noindent\textit{(ii) Broadcast lower bound.}
Consider any common broadcast sequence $\Gamma=(z_1,\dots,z_T)$ that acquires
$g$ for every mind. Fix $i$. Before mind $\mind_i$ can acquire $g$, it
must first acquire all $L-1$ private prerequisite concepts
\[
p_{i,1},\dots,p_{i,L-1}.
\]
Moreover, if $i'\neq i$, then none of the concepts $p_{i,j}$ lies in
$\mc U_{\mind_{i'}}$. Hence a broadcast signal targeting $p_{i,j}$ can
help at most mind $\mind_i$; it produces no acquisition for any other mind.

It follows that at least $L-1$ rounds must be devoted to the private
prerequisites of each mind $i$. Summing over $i=1,\dots,k$, at least $
k(L-1)$
rounds are required to make all minds ready for a signal targeting $g$.

Finally, one additional round targeting $g$ is necessary, since $g\notin
\Acal$ and is acquired only when a signal with target $g$ is parseable.
Hence
\[
T\ge k(L-1)+1.
\]

\medskip
\noindent\textit{(iii) Tightness.}
Consider the broadcast sequence
\[
z_{1,1},\dots,z_{1,L-1},
z_{2,1},\dots,z_{2,L-1},
\dots,
z_{k,1},\dots,z_{k,L-1},
z_g.
\]
During the block $z_{i,1},\dots,z_{i,L-1}$, only mind $\mind_i$ advances;
all other minds ignore those signals. After the first $k(L-1)$ rounds, each
mind $\mind_i$ has acquired $p_{i,L-1}$. The final signal $z_g$ is
therefore parseable for every mind, so all of them acquire $g$ on the last
round. Thus the lower bound is attained.
\qed

\section{Additional results}
This appendix collects supplementary results that are invoked in the proofs above but are not essential to the main narrative. It also records additional consequences of the framework that may be of independent interest.

\begin{corollary}[Not every union-closed family above the axioms is a learning space]\label{cor:not-every-knowledge-space}
The class of $\Acal$-based learning spaces on $\mc U$ is a strict subclass of the class of union-closed families $
\mathbb F\subseteq 2^{\mc U}$.
\end{corollary}

\begin{proof}
Every $\Acal$-based learning space is, by definition, union-closed and lies above $\Acal$, so only strictness needs to be shown. Let
\[
\mc U=\{a,b\},
\qquad
\Acal=\varnothing,
\qquad
\mathbb F=\{\varnothing,\{a,b\}\}.
\]
Then $\mathbb F$ is union-closed and contains $\Acal$. However, it fails accessibility, since neither
\[
\{a,b\}\setminus\{a\}=\{b\}
\qquad\text{nor}\qquad
\{a,b\}\setminus\{b\}=\{a\}
\]
belongs to $\mathbb F$. Hence $\mathbb F$ is not an $\Acal$-based learning space.
\end{proof}
\begin{lemma}[Prefix closure of reachable acquired concept sets]\label{lem:prefix}
If $\mc K\in \mathbb K_\mind$ and $
\Acal_\mind=\mc K_0 \subset \mc K_1 \subset \cdots \subset \mc K_L = \mc K$
is a witnessing chain, then every intermediate set $\mc K_i$ also belongs to $\mathbb K_\mind$.
\end{lemma}
\begin{proof}
Each $\mc K_i$ is reachable from $\Acal_\mind$ by truncating the witnessing chain at step $i$.
\end{proof}

\begin{proposition}[Parseability preserves information]\label{prop:parseable-positive}
Let
\[
U_{t+1}=\{\tgt(Z_{t+1})\notin \Phi_\mind(\mc K_t)\}.
\]
Then
\[
\mathds I(\Theta;Y_{t+1}\mid \Fcal_t,U_{t+1}^c)
=
\mathds I(\Theta;Z_{t+1}\mid \Fcal_t,U_{t+1}^c).
\]
In particular, if the right-hand side is strictly positive, then so is the left-hand side.
\end{proposition}

\begin{proof}
On $U_{t+1}^c$, the parser acts as the identity, so
\[
Y_{t+1}=Z_{t+1}
\qquad\text{almost surely.}
\]
The identity of the conditional mutual informations follows immediately.
\end{proof}

\begin{corollary}[Unlimited rephrasing can be useless under sharp parsing]\label{cor:rephrasing}
Fix time $t$ and a mind $\mind$. Let $
U_t(c)=\{c\notin \Phi_\mind(\mc K_t)\}$.
Let $(Z^{(j)}_{t+1})_{j\ge 1}$ be any family of $\mc Z$-valued random variables such that
\[
\tgt(Z^{(j)}_{t+1})=c
\qquad\text{almost surely for every }j\ge 1,
\]
and define $
Y^{(j)}_{t+1}=\rho_\mind(Z^{(j)}_{t+1},\mc K_t)$.
Then for every $j\ge 1$, $
\mathds I(\Theta;Y^{(j)}_{t+1}\mid \Fcal_t)=0$ \text{almost surely on} $U_t(c)$.
\end{corollary}
\begin{proof}
On $U_t(c)$, the targeted concept is not ordered, so $
Y^{(j)}_{t+1}=\bot$ \text{almost surely.}
Hence $Y^{(j)}_{t+1}$ is conditionally constant given $\Fcal_t$, so the conditional mutual information is zero.
\end{proof}

\bibliographystyle{plainnat}
\bibliography{main}

\end{document}